%% file: ex_article.tex
\documentclass[final,hidelinks,onefignum,onetabnum]{siamart251216}
\usepackage[utf8]{inputenc}
\input{ex_shared}

\ifpdf
\hypersetup{
  pdftitle={SD-FSNN},
  pdfauthor={Z.Y. Liang}
}
\fi

\begin{document}

\maketitle

% REQUIRED
\begin{abstract}
This paper introduces the Stochastic-Dimension Frozen Sampled Neural Network (SD-FSNN), a novel computational framework for solving the high-dimensional Gross-Pitaevskii equation (GPE) on unbounded domains. The proposed method circumvents the curse-of-dimensionality that plagues traditional discretizations and the computational bottlenecks of gradient-based neural network solvers through a synergistic combination of techniques. First, a prescribed Gaussian envelope encodes the far-field decay of the wavefunction, enabling a space-time separation where the spatial approximation is handled by a frozen, single-hidden-layer neural network with data-driven sampled features. This yields a gradient-free formalism where spatial derivatives are analytically precomputed and time-dependence is evolved via reduced ODEs. Second, a stochastic-dimension sampler provides a conditionally unbiased estimate of the spatial operator by evaluating only a small subset of spatial dimensions at each time step, essentially reducing computational and memory costs. Discrete conservation laws are also enforced, ensuring long-term stability. Extensive numerical experiments on GPE in up to 1000 dimensions demonstrate that SD-FSNN achieves significantly higher accuracy and efficiency compared to state-of-the-art methods, including PINNs, randomized feature methods, and tensor-network approaches. The results confirm that SD-FSNN effectively mitigates the Kolmogorov $n$-width barrier for frozen-basis models on structured solution manifolds.
%We propose a stochastic-dimension frozen sampled neural network (SD-FSNN) for solving the Gross--Pitaevskii equation (GPE) on unbounded domains.
%The stochastic-dimension sampler is conditionally unbiased for the additive spatial operator at fixed collocation points, frozen features, and coefficient state. Its dominant operator-evaluation cost scales with the dimension-batch size rather than the full ambient dimension.
%The hidden weights and biases are sampled and then frozen, avoiding iterative gradient-based optimization and thereby reducing training cost.
%A space-time separation strategy combined with adaptive ordinary differential equation (ODE) solvers updates the time-dependent coefficients while preserving temporal causality.
%To respect the GPE structure, we incorporate a Gaussian-weighted ansatz for exponential decay at infinity, together with a closed-form two-step projection that exactly preserves the discrete mass and the quadratic part of the discrete Hamiltonian, while the quartic interaction term is monitored.
%For the problems considered, SD-FSNN reduces the dominant operator-evaluation cost from linear in $d$ to linear in the dimension-batch size.
%Numerical experiments show that SD-FSNN achieves lower errors and shorter runtimes than the tested baselines. The benchmarks span a range of spatial dimensions and interaction parameters.
\end{abstract}

% REQUIRED

% REQUIRED

\section{Introduction}
\label{sec:intro}

The Gross--Pitaevskii equation (GPE), a  nonlinear Schr\"odinger equation, is of paramount importance in mathematics and physics.
Its dimensionless form in $d$ dimensions ($d\in \mathbb{N}_+$) reads \cite{GPEreview,erdHos2010derivation,lieb2005math}
\begin{equation}\label{eqn:GPE}
\left\{
\begin{aligned}
& i\,\partial_{t}\psi(\bm{x},t)=-\frac{1}{2}\nabla^2\psi(\bm{x},t)+V_d(\bm{x})\psi(\bm{x},t)+\beta\,|\psi(\bm{x},t)|^{2}\psi(\bm{x},t),\quad \bm{x}\in\mathbb{R}^{d},\ t>0,\\[2mm]
& \psi(\bm{x},0)=\psi_0(\bm{x}),\quad
\bm{x}\in\mathbb{R}^{d},
\end{aligned}
\right.
\end{equation}
where $\psi$ is the complex-valued unknown wave function with the initial data $\psi_0$, 
$\mathbf{x}=(x_1,\ldots,x_d)$, 
$
  V_d(\bm{x}) = \frac{1}{2}\sum_{j=1}^{d} \gamma_j^2 x_j^2$ is a given potential %and $
 % \beta_d = \beta \prod_{j=d+1}^{n} \sqrt{\frac{\gamma_j}{2\pi}}$
with $\gamma_j > 0$ the trapping frequencies,  and $\beta$ denotes the strength of nonlinear self-interaction. 
For $d\leq3$, GPE \eqref{eqn:GPE} has been justified \cite{erdHos2010derivation,lieb2005math} as a model with high fidelity to describe the Bose--Einstein condensates in the low-temperature regime \cite{anglin2002bose}. Mathematically, it is of interest in arbitrary dimensions, whose well-posedness has been established \cite{Cazenave}. Particularly, the defocusing case ($\beta>0$) of GPE \eqref{eqn:GPE} is globally well-posed.  
It has also served as a benchmark for emerging quantum simulation algorithms \cite{georgescu2014quantum}. 
The dynamics of \eqref{eqn:GPE} conserves in time the \textit{mass} (normalized for simplicity):
\begin{equation}
\label{eqn:norm}
N(\psi):=\|\psi(\cdot,t)\|^2=\int_{{\mathbb{R}}^d} \; |\psi(\bm{x},t)|^2\;d\bm{x}
\equiv \int_{{\mathbb{R}}^d} \; |\psi_0(\bm{x})|^2\;d\bm{x}=1,\quad \forall t\geq0,
\end{equation}
and the \textit{energy/Hamiltonian}:
\begin{equation}\label{eqn:energy}
E(\psi) := \int_{\mathbb{R}^d} \left[\frac{1}{2}|\nabla\psi(\bm{x},t)|^2 + V_d(\bm{x})|\psi(\bm{x},t)|^2 + \frac{\beta}{2}|\psi(\bm{x},t)|^4\right] d\bm{x}\equiv E(\psi_0), \quad \forall\, t \geq 0.
\end{equation}

Extensive numerical methods have been proposed to solve GPE \eqref{eqn:GPE} in the literature. 
Due to fast decay of the solution at far field \cite{lieb2005math}, 
a large class of traditional discretization methods was designed 
on truncated domains with hard (zero or periodic) boundary conditions imposed 
\cite{GPEreview,fem,DG}. Their validity, however, is restricted by the chosen domain size, where spurious reflections would occur as waves reach the boundary. Such limitations motivate on the one hand the design of artificial boundary conditions \cite{PMLreview,xu2006absorb,zhang2008unif} for the truncated domain, while their accuracy is often sensitive to some chosen parameters and/or limited to the linear case. On the other hand, spectral methods that naturally account for the decay at infinity, e.g., Hermite or Laguerre basis \cite{spectralbook}, enable accurate computations directly in whole space \cite{bao2005fourth,chou2023adaptive,Yu,ShenFOCM}, though 
those whole-space basis lack a natural fast transform like FFT \cite{spectralbook}.
%These adaptive schemes incorporate scaling adjustments and $p$-adaptivity, improving spectral resolution across varying decay rates and solution behaviors \cite{chou2023adaptive}.
Despite continuing developments, 
traditional methods all suffer from curse-of-dimensionality (CoD), where computational and memory costs grow exponentially with dimension $d$. 
%The number of tensor-product basis functions  grows exponentially with the ambient dimension $d$, leading to prohibitive computational and memory costs. 
%As a result, traditional spectral methods become impractical for high-dimensional GPEs arising in contemporary experimental contexts.
%Relevant examples include dipolar condensates \cite{lahaye2009physics} and quantum simulation platforms with tens of spatial modes \cite{wang2017exp}. %On a tensor product grid with $N+1$ points per dimension, storing the solution requires $O(N^d)$ memory.
%The dominant per-step cost comes from forward and inverse transforms between physical and spectral spaces, which scale as $O(dN^{d+1})$ \cite{bao2005fourth}.
%Thus, for fixed resolution per dimension, both memory and computational cost grow exponentially with the spatial dimension $d$.

To mitigate CoD, 
 deep-learning-based solvers for high-dimensional PDEs have been developed along three main directions, depending on how differential operators are handled. 
The first direction solves PDEs using their stochastic representations, e.g., DeepBSDE \cite{beck2019machine,han2018hjb}, deep splitting \cite{beck2021deep}, FBSNNs \cite{raissi2024forward}, multilevel Picard schemes \cite{beck2024overcome}, TNN \cite{wang2025tensor,wang2022tensor} and separable PINNs \cite{cho2022spinn}. 
The second direction reduces operator evaluation cost via randomized estimators, e.g., SDGD \cite{hu2024tackling}, HTE \cite{hu2024hte}, RS-PINNs \cite{hu2025smooth}, Score-PINNs \cite{hu2025score} and STDEs \cite{shi2024stde}. 
The third direction avoids backpropagation and uses random features or analytic fitting, represented by RFM \cite{rahimi2007random}, LocalELM \cite{huang2012extreme} and GNN \cite{ainsworth2021gnn,ainsworth2025egnn}. 
Related methods that reduce reliance on full backpropagation include CAN-PINNs \cite{chiu2022can}, HFD-PINNs \cite{lv2023hybrid}, fPINNs \cite{pang2019fpinns} and the Deep Galerkin method \cite{sirignano2018dgm}. 
Despite these advances, solving high-dimensional GPE on an unbounded domain has not been addressed before and remains challenging under state-of-the-art methods, which shall be seen in our numerical experiments. 
Randomized gradient estimators may introduce bias through nonlinear losses or suffer from dimension-dependent variance amplification, causing spurious dissipation. 
Gradient-based training can be slow, unstable, and sensitive to hyperparameters. 
Moreover, many standard neural ansatz lacks built-in spatial localization, making it difficult to learn without artificial boundaries, and 
global-in-time training or weak causal control may further amplify errors. 
%Mass and energy conservation in high-dimensional settings has received limited attention.%, and violations can lead to numerical dissipation and loss of physical structure \cite{liang2026stochastic}.

In this work, we shall develop a \textbf{S}tochastic-\textbf{D}imension \textbf{F}rozen \textbf{S}ampled \textbf{N}eural \textbf{N}etwork  (\textbf{SD-FSNN}) for solving high-dimensional GPE on unbounded domain, addressing the aforementioned challenges.
Our main contributions include
\begin{itemize}
    \item \textbf{De-enveloped network and featured basis.}
    We construct a novel network with a prescribed exponential-decay envelope and sample the featured basis for spatial approximation of the de-enveloped part. This leads to a gradient-free method for GPE in whole space at high accuracy and  alleviates the Kolmogorov n-width barrier \cite{peherstorfer2022breaking}. 

    \item \textbf{Space-time separation and stochastic-dimension sampler.}
    The proposed network gives rise
    to a space-time separation, where time-dependence is computed on reduced ODEs. 
    The spatial part involves only  differentiations in additive manner, and we propose a sampler with respect to spatial dimensions for evaluations per time level. This reduces essentially the computational and memory costs with theoretically proved 
    conditional unbiasedness of the stochastic-dimension sampler.
    %This enforces temporal causality by construction, avoids recomputing high-dimensional spatial features during time stepping, and facilitates long-time prediction.

    %\item \textbf{Stochastic-dimension sampler.}
    %Thanks to the additive differentiations resulting from our space-time separation, 

 %   \item \textbf{Gradient-free randomized representation.}
  %  We employ adaptively sampled weights to obtain a randomized representation that avoids gradient-based optimization.
   % This choice alleviates the slow convergence and limited accuracy often observed in gradient-based solvers for high-dimensional PDEs.

    \item \textbf{Structure-preserving constraints.}
In addition to the enforced exponential decay at infinity, 
    we also embed a closed-form two-step projection on the evolving temporal coefficients, which simultaneously preserves the mass and energy at the discrete level. This further enhances long-term performance with uniform-in-time boundedeness of the dynamical coefficients proved. 

   % \item \textbf{Validation across dimensions and interaction.}
    
\end{itemize}
Through extensive numerical experiments, we evaluate the effectiveness and robustness of SD-FSNN on GPE across a range of spatial dimensions.
    Numerical results demonstrate that SD-FSNN compares favorably with several representative high-dimensional PDE solvers on the benchmarks considered.

The rest of the paper is organized as follows. 
\Cref{sec:problem} provides preliminaries regarding computations of high-dimensional PDEs. 
\Cref{sec:sdfsnn} presents SD-FSNN and its analysis. The efficiency and accuracy of SD-FSNN is demonstrated numerically in \Cref{sec:exp}, and comparisons with existing approaches are made in
\Cref{sec:discussion}. \Cref{sec:conclusion} draws concluding remarks.

\section{Preliminaries}
\label{sec:problem}
This section provides some pre-knowledge related to trainings of high-dimensional PDEs and a brief review of existing works proposed from two aspects, which will be compared numerically. 
%\subsection{Problem formulation}

%The energy is conserved during time evolution,
%\begin{equation}\label{eqn:energy_cons}
%E(\psi(\cdot,t)) \equiv E(\psi_0), \quad \forall\, t \geq 0.
%\end{equation}
%This conservation law follows from the Hamiltonian structure of the GPEs associated with \cref{eqn:energy} under the constraint \cref{eqn:norm}.
% The energy consists of three terms: the kinetic energy ($\frac{1}{2}|\nabla\psi|^2$), the potential energy ($V_d|\psi|^2$), and the interaction energy ($\frac{\beta_d}{2}|\psi|^4$).

%Since the equation is posed on the whole space, Hermite basis functions are natural and have been successfully used for related equations \cite{guo2003spec}.
%On a tensor product grid with $N+1$ points per dimension, storing the solution requires $O(N^d)$ memory.
%The dominant per-step cost comes from forward and inverse transforms between physical and spectral spaces, which scale as $O(dN^{d+1})$ \cite{bao2005fourth}.
%Thus, for fixed resolution per dimension, both memory and computational cost grow exponentially with the spatial dimension $d$.

\subsection{Randomized gradient estimators}
For trainings of high-dimensional or high-order PDEs via gradient-based optimizers, 
automatic differentiation (AD) becomes expensive, since forward-mode AD scales with $d$ and nested backward-mode AD rapidly enlarges computational graphs. To reduce the burden, different
randomized gradient estimators by sampling coordinates, directions, or perturbations have been proposed in the literature:
\begin{itemize}
    \item \textbf{Stochastic dimension gradient descent (SDGD).}
    SDGD \cite{hu2024tackling} approximates a summation of partial derivatives by some rescaled one over a random subset, reducing explicit $d$-dependence but still using AD for the sampled.

    \item \textbf{Randomized smoothing PINNs (RS-PINNs).}
    RS-PINNs \cite{hu2025smooth} replace high-order derivatives with Gaussian-smoothed function evaluations, which avoids high-order AD but may introduce bias from nonlinear residuals. Debiasing can increase variance.

    \item \textbf{Hutchinson trace estimation (HTE).}
    HTE \cite{hu2024hte} estimates Laplacian % $\Delta u = \operatorname{tr}(\nabla^2 u)$ 
via some random Hessian-vector products, %$v^\top \nabla^2 u\, v$.
    avoiding the full Hessian. %, but is not naturally suited to general higher-order operators.

    \item \textbf{Stochastic Taylor derivative estimator (STDE).}
    STDE \cite{shi2024stde} estimates differential operators via random sparse Taylor-mode jets with debiasing scalars, which still requires Taylor-mode AD and gradient-based training.

    \item \textbf{Forward Laplacian (ForLap).}
    ForLap \cite{li2024forlap} streamlines AD-based Laplacian evaluation with $\mathcal{O}(d)$ complexity.
\end{itemize}
Although these methods are helpful in reducing derivative-evaluation costs, they still essentially rely on AD, trainable networks, or iterative optimization. 
This motivates us to consider a gradient-free formulation. % with analytical feature derivatives and stochastic-dimension spatial-operator samplers.

% \subsubsection{Gradient-free solvers}
\subsection{Randomized neural networks}

Randomized neural networks (RaNNs) \cite{gallicchio2020rann,huybrechs2024rann} reduce training cost by freezing hidden features and fitting only output coefficients. Representative works include RFM \cite{rahimi2007random} and LocalELM \cite{huang2012extreme}, briefly presented in the following. 

Given random features $\varphi_1, \ldots, \varphi_N$ according to Gaussian distribution on some domain of $\mathbb{R}^d$, RaNNs construct the approximation as
\begin{equation}\label{eqn:rann}
    U_W: \mathbb{R}^d \to \mathbb{R}: \mathbf{x}\mapsto \sum_{j=1}^N w_j \varphi_j(\mathbf{x};\theta_j),
\end{equation}
where only $w_j$ are fitted and $\theta_j$ are typically the sample parameters. 
RFM specializes \cref{eqn:rann} to a single hidden layer with frozen parameters,
\begin{equation}\label{eqn:rfm}
    U_W^{A,B}: \mathbb{R}^d \to \mathbb{R}: x\mapsto \sum_{j=1}^N w_j \sigma(A_j\cdot x+B_j),
\end{equation}
where $\sigma$ is nonlinear and $A_i,B_i$ are random weights and biases. 
Given the representation \cref{eqn:rfm}, LocalELM determines the weights $w_j$ analytically from the given data $\{(x_k,T_k)\}_{k=1}^n$ by imposing
$\sum_{j=1}^N w_j \sigma(A_j\cdot x_k +B_j) = T_k, \  1\leq k\leq n$,
where the matrix form $Hw=T$ with $T=(T_k)_{n\times1},\,w=(w_j)_{N\times1},\,H=(H_{kj})_{n\times N},\,H_{kj}=\sigma(A_j\cdot x_k+B_j)$ gives the minimum-norm solution $w=H^\dagger T$. Here and after, $(\cdot)^\dagger$ denotes the pseudo-inverse.

Although efficient, RFM and LocalELM lack causal time-marching and may accumulate long-term errors.

\section{Stochastic-dimension frozen sampled neural network (SD-FSNN)}
\label{sec:sdfsnn}

In this section, we introduce our SD-FSNN method for high-dimensional GPE \eqref{eqn:GPE} on unbounded domain.
As illustrated in \Cref{fig:sdfsnn_core_idea}, SD-FSNN consists of two main components:
i) A fundamental network architecture that 
constructs tailored frozen spatial basis with a decaying envelope and evolves time-dependent coefficients through ODEs, for high accuracy and efficiency; 
ii) An add-on unbiased sampling of spatial dimension to evaluate differential operators, further reducing the high-dimensional operator-evaluation cost.
%ii) frozen spatial bases are constructed by pairwise sampled initialization and multiplied by an exponential-decay envelope to encode whole-space boundary behavior;
%iii) space and time are separated by freezing the hidden weights and evolving only the time-dependent output coefficients through an ODE system;
%iv) the spatiotemporal approximation $\hat{\psi}(\bm{x},t)$ is reconstructed from the evolving coefficients and the frozen spatial bases at the prescribed evaluation times.

\begin{figure}[htbp]
    \centering
    \includegraphics[width=\textwidth]{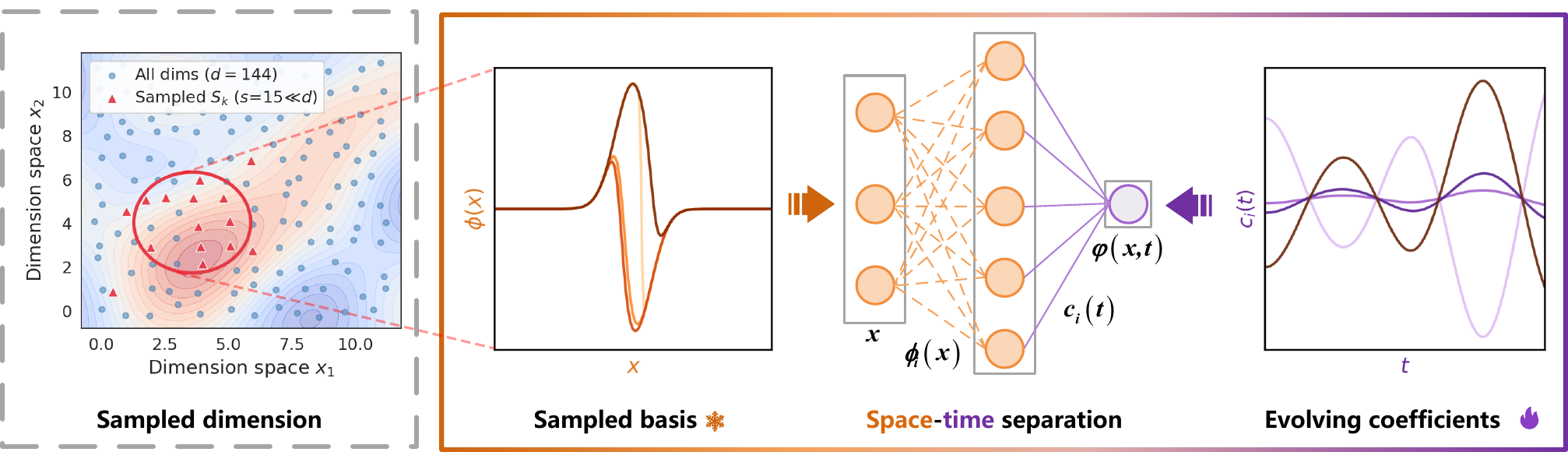}
    \vspace{-1.5em}
    \caption{\textbf{Core ideas of SD-FSNN}: SD-FSNN avoids backpropagation through space-time separation and random spatial sampling. Frozen basis functions with exponential decay encode far-field behavior. Differential operators are evaluated on a sampled subset of dimensions. Time-dependent coefficients are evolved on ODEs.}
    \label{fig:sdfsnn_core_idea}
\end{figure}

\subsection{Fundamental architecture}
We first present the fundamental architecture of SD-FSNN for separated space-time evaluations. 
\subsubsection{Network ansatz and featured basis} 
To encode far-field decay, we factor the wave function as
\begin{equation}\label{ansatz}
\psi(\bm{x},t)=\rho(\bm{x})u(\bm{x},t),
\end{equation}
where $\rho:\mathbb{R}^d\to(0,\infty)$ is a prescribed normalized decaying envelope $\rho(\bm{x})=(\alpha/\pi)^{d/2}\cdot\exp(-\alpha\|\bm{x}\|^2)$ with $\alpha>0$, so that $\int_{\mathbb{R}^d}\rho(\bm{x})\,d\bm{x}=1$. 
The prefactor enforces normalization, while $\alpha$ controls the decay rate which we shall infer from the initial data. 
% For the isotropic Gaussian state $\psi_0(\bm{x})=\pi^{-d/4}e^{-\|\bm{x}\|^2/2}$, one sets $\alpha=0.5$. 
This provides small values and so small errors at the far field.
% where $\rho:\mathbb{R}^d\to(0,\infty)$ is a prescribed decaying envelope $\rho(\bm{x})=\exp(-\alpha\|\bm{x}\|^2)$ with $\alpha>0$ chosen for normalization  $\int_{\mathbb{R}^d}\rho(\bm{x})d\bm{x}=1$. This provides small values and so small errors at the far field. 
The de-enveloped variable is $u=\rho^{-1}\psi$ with initial data $u_0=\rho^{-1}\psi_0$, and the neural network needs only to approximate this part where
%Throughout the paper, the spatial variable $\bm{x}\in\mathbb{R}^{1\times d}$ is treated as a \emph{row} vector, so that $\bm{x}^\top\in\mathbb{R}^{d\times 1}$ is the corresponding column vector.
we use a single-hidden-layer frozen feature approximation:
\begin{equation}\label{eqn:ansatz}
   \hat{u}(\bm{x},t) = {c}(t)\Phi(\bm{x})^\top + c_0(t),\quad \Phi(\bm{x}) = [\phi_1(\bm{x}), \dots, \phi_M(\bm{x})]\in\mathbb{R}^{1\times M},\quad M\in\mathbb{N}_+.
\end{equation}
Here, ${c}(t)\in\mathbb{C}^{1 \times M},\,c_0(t)\in\mathbb{C}$ are evolving coefficients and each feature is taken as $\phi_m(\bm{x})=\tanh(w_m\bm{x}^\top+b_m)$  with a frozen weight vector $w_m\in\mathbb{R}^{1\times d}$ and a frozen bias $b_m\in\mathbb{R}$ to be chosen. %, so that $w_m\bm{x}^\top\in\mathbb{R}$.
Let $W\in\mathbb{R}^{M\times d}$ be the row-stack of $w_1,\ldots,w_M$, $b=[b_1,\ldots,b_M]^\top\in\mathbb{R}^{M\times 1}$, $C(t)=[c(t),c_{0}(t)]\in\mathbb{C}^{1\times(M+1)}$ and $\Psi(\bm{x}):=[\Phi(\bm{x}),1]^\top\in\mathbb{R}^{(M+1)\times 1}$, then \eqref{eqn:ansatz} in matrix form reads $\hat{u}(\bm{x},t)=C(t)\Psi(\bm{x})$
and the solution of \eqref{eqn:GPE} is  approximated by $\hat{\psi}(\bm{x},t)=\rho(\bm{x})\hat{u}(\bm{x},t)$ with time dependence through $C(t)$.

\begin{figure}[h!]
    \centering
    \includegraphics[width=\textwidth]{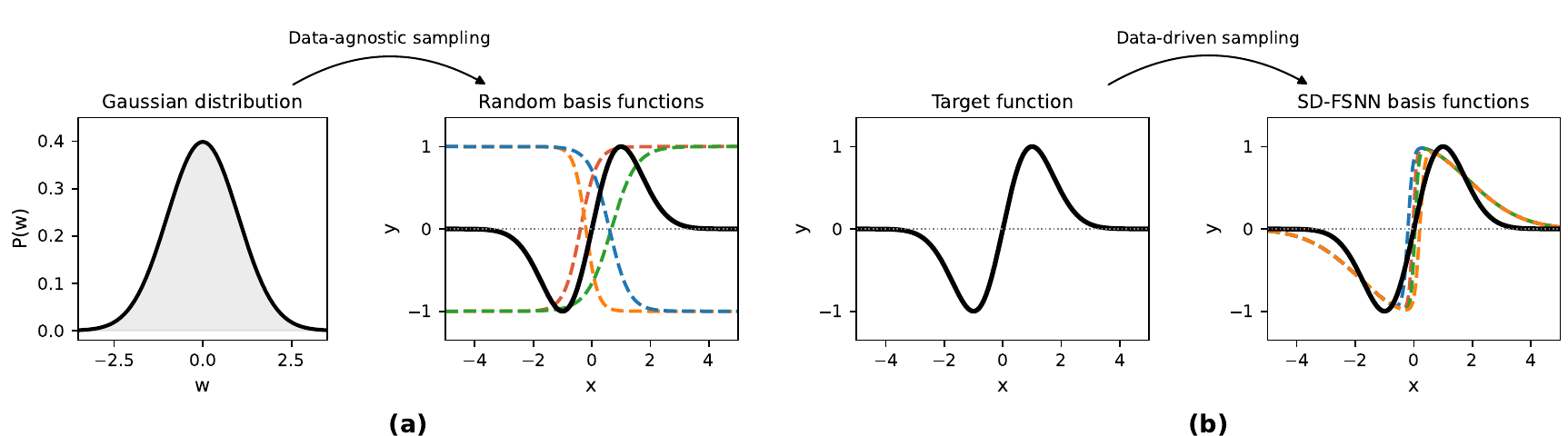}
    \vspace{-1.5em}
    \caption{Sampling strategy illustration: (a) Data-agnostic sampling; (b) Data-driven sampling.}
    \label{fig:sdfsnn-basis}
\end{figure}
For values of $w_m$ and $b_m$,
%Let $p_\rho(\bm{x})=Z_\rho^{-1}\rho(\bm{x})^2$ be the normalized envelope-induced density, with normalization constant $Z_\rho=\int_{\mathbb{R}^d}\rho(\bm{x})^2\,d\bm{x}<\infty$. 
following \cite{bolager2023sampling}, we sample two points $\bm{x}_{m,1}, \bm{x}_{m,2} \in \mathbb{R}^d$ for each $1\leq m\leq M$, and set
\begin{align}\label{eqn:sample_weight}
     w_m = s_1 \frac{\bm{x}_{m,2} - \bm{x}_{m,1}}{\lVert \bm{x}_{m,2} - \bm{x}_{m,1}\rVert^{2}}, \quad b_m = - 
     w_m\bm{x}_{m,1}^\top
     + s_2,
\end{align}
where %$\langle\cdot,\cdot\rangle$ denotes the Euclidean inner product on $\mathbb{R}^d$, and
the constants $s_1,s_2\in\mathbb{R}$ are chosen so that the feature interpolates between two pre-specified activation levels $\sigma_-<\sigma_+$ at the two sampled points. Here we consider $\sigma_\pm=\pm 0.5$, which lies in the well-conditioned, non-saturated regime of $\tanh$ \cite{bolager2023sampling}. This gives $s_1=2\tanh^{-1}(0.5)$ and $s_2=-\tanh^{-1}(0.5)$ with $\phi_m(\bm{x}_{m,1})=-0.5$ and $\phi_m(\bm{x}_{m,2})=0.5$, making
 each $\tanh$ neuron geometrically a soft step centered on $\tfrac12(\bm{x}_{m,1}+\bm{x}_{m,2})$, oriented along $\bm{x}_{m,2}-\bm{x}_{m,1}$ with a transition width proportional to $\lVert\bm{x}_{m,2}-\bm{x}_{m,1}\rVert$. This construction binds the feature of each basis directly to the sampled pair. 
For the sampling of the pair $(\bm{x}_{m,1},\bm{x}_{m,2})$, we adopt 
the \emph{data-driven} sampling (\Cref{fig:sdfsnn-basis}(b)) guided by a probability: $$\rho(\bm{x}_{m,1})\rho(\bm{x}_{m,2})\|u_0(\bm{x}_{m,2})-u_0(\bm{x}_{m,1})\|/\|\bm{x}_{m,2}-\bm{x}_{m,1}\|.$$ %concentrating basis transitions near regions of rapid variation weighted by the envelope 
It is motivated by the fact that $\nabla u_0$ mainly reflects the internal variation and $\rho$ can direct basis transitions to where the solution structure needs to be resolved. This is important for both accuracy and efficiency. %\cite{bolager2023sampling,datar2024fast} 
%The weights and biases are then set so that the $\tanh$ output equals $-0.5$ at $\bm{x}_{(1)}$ and $+0.5$ at $\bm{x}_{(2)}$, aligning activation centers with the direction $\bm{x}_{(1)}\!\to\!\bm{x}_{(2)}$.
In contrast, LocalELM and RFM typically use data-agnostic sampling 
with weights drawn from a prescribed (e.g., Gaussian) distribution, producing basis with transitions not adapted to the target function (\Cref{fig:sdfsnn-basis}(a)). 

%The envelope factorization and gradient-guided sampling jointly address two difficulties of whole-space GPEs. First, on $\mathbb{R}^d$ no uniform measure exists and the normalized density $p_\rho(\bm{x}):=Z_\rho^{-1}\rho(\bm{x})^2$ (with $Z_\rho=\int_{\mathbb{R}^d}\rho^2\,d\bm{x}$) supplies a base measure concentrated on the physical support. Second, 

%The rule \cref{eqn:sample_weight} avoids both inefficiencies by coupling the base measure to $\rho$ and the reweighting to $\|\nabla u_0\|$ (\Cref{fig:sdfsnn-basis}(b)).

%The resulting SD-FSNN architecture is illustrated in \Cref{fig:sdfsnn-net} and consists of five conceptual layers.
%The \emph{input layer} takes $N_c$ collocation points $X\in\mathbb{R}^{N_c\times d}$.
%The \emph{sampled basis functions} layer then evaluates the frozen features $\Phi(X)\in\mathbb{R}^{N_c\times M}$ constructed via pairwise sampling \eqref{eqn:sample_weight}.
%$Next, the \emph{projection layer} augments the basis with a constant column to form $\Psi(X)=%[\Phi(X),\mathbf{1}_{N_c}]^\top\in\mathbb{R}^{(M+1)\times N_c}$.
%The \emph{SVD layer} compresses $\Psi(X)$ via a truncated SVD to obtain the orthogonalized reduced basis $\Psi_r(X)=U_r^\top\Psi(X)\in\mathbb{R}^{r\times N_c}$.
%Finally, the \emph{output layer} produces the approximation $C(t)\Psi_r(X)\in\mathbb{C}^{1\times N_c}$, where $C(t)\in\mathbb{C}^{1\times r}$ are the time-dependent coefficients evolved by an adaptive ODE solver.
%All hidden-layer parameters (weights, biases, and SVD factors) are frozen after initialization. Only $C(t)$ is updated during time integration.

\subsubsection{Separated space-time evaluation}
\label{sec:separat} 
Plugging \eqref{ansatz} into \eqref{eqn:GPE} and noting that $\rho$ is time-independent yields 
\begin{equation}\label{eqn:deenv}
\partial_t u=i\left[\tfrac12\nabla^2 u+(\nabla\rho/\rho)\cdot\nabla u+\tfrac12(\nabla^2\rho/\rho)u -V_du-\beta\rho^2|u|^2u\right].
\end{equation}
We denote the envelope-induced drift and Laplacian-ratio coefficients as
\begin{equation}\label{eqn:envelope_coeffs}
 a_j(\bm{x}):=\frac{\partial_{x_j}\rho(\bm{x})}{\rho(\bm{x})},\quad q_j(\bm{x}):=\frac{\partial_{x_j}^2\rho(\bm{x})}{\rho(\bm{x})},\quad q(\bm{x}):=\sum_{j=1}^d q_j(\bm{x})=\frac{\nabla^2\rho(\bm{x})}{\rho(\bm{x})}.
\end{equation}
%we can write the spatial operator in \cref{eqn:deenv} in the dimension-wise additive form used later in \Cref{sec:sdfsnn_weights}.

Substituting the ansatz \cref{eqn:ansatz} into \cref{eqn:deenv}, we take $N_c\in\mathbb{N}_+$ collocation points by importance sampling to give a projected least-squares ODE for $C(t)$.
%This reformulation preserves the inherent causal structure of time-dependent PDEs.
Assemble $N_c$ collocation points in a matrix $X\in\mathbb{R}^{N_c\times d}$, one point per row %(so that $\bm{x}\in\mathbb{R}^{1\times d}$ is a single point and $X$ its multi-point stack), 
and overload $\Phi,\Psi,\rho,V_d$ to accept  $\Phi(X)\in\mathbb{R}^{N_c\times M}$, $\Psi(X):=[\Phi(X),\mathbf{1}_{N_c}]^\top\in\mathbb{R}^{(M+1)\times N_c}$, and $\rho(X),V_d(X)\in\mathbb{R}^{N_c\times 1}$.
By letting $U(X,t):=C(t)\Psi(X)\in\mathbb{C}^{1\times N_c}$ %; the row-wise derivatives $\partial_{x_j}\Psi(X),\nabla^2\Psi(X)\in\mathbb{R}^{(M+1)\times N_c}$ have a zero last row (constant feature), with closed forms in \cref{eqn:ansatz_x,eqn:laplacian}.
 and denoting $\odot$ the entry-wise (Hadamard) product, we get the complex-valued projected ODEs:
\begin{align}
        \dot{C}(t) &= R_\rho(X,C(t))\Psi(X)^\dagger,\mbox{ with }\label{eqn:ode} \\
        R_\rho(X,C(t)) &= i\Big[\tfrac12 C(t)\nabla^2\Psi(X)
        +\sum_{j=1}^d a_j(X)^\top\odot\big(C(t)\partial_{x_j}\Psi(X)\big)\nonumber\\
        &\quad\ \ +\tfrac12 q(X)^\top\odot U(X,t)
        -V_d(X)^\top\odot U(X,t)
        \nonumber\\
        &\quad\ \ -\beta(\rho(X)^2)^\top\odot |U(X,t)|^2\odot U(X,t)\Big],\nonumber
\end{align}
where $a_j(X), q(X)\in\mathbb{R}^{N_c\times 1}$ are pointwise evaluations of $a_j$ and $q$ in \cref{eqn:envelope_coeffs}. %and $(\cdot)^\dagger$ denotes the Moore--Penrose pseudo-inverse.
The formulation \eqref{eqn:ode} is exact at the location level whenever its right-hand side is in the row space of $\Psi(X)$, 
otherwise $(\cdot)\Psi(X)^\dagger$ acts as the projection of the least-squares. % onto the row space of $\Psi(X)$.
%Since standard ODE solvers are designed for real-valued systems, we decompose the complex-valued parameters as $C(t) = C_{\text{re}}(t) + iC_{\text{im}}(t)$ with $C_{\text{re}}(t), C_{\text{im}}(t) \in \mathbb{R}^{1\times(M+1)}$.
%This decomposition yields the coupled real-valued ODE system
%\begin{equation}\label{eqn:ode_system}
 %   \frac{d}{dt}\begin{bmatrix} C_{\text{re}}(t) \\ C_{\text{im}}(t) \end{bmatrix}
%    = \begin{bmatrix} \text{Re}\left(R_\rho(X,C(t))\Psi(X)^\dagger\right) \\ \text{Im}\left(R_\rho(X,C(t))\Psi(X)^\dagger\right) \end{bmatrix}.
%\end{equation}
The initial condition of \cref{eqn:ode} is computed via a least-squares solution
\begin{equation}\label{eqn:C0}
C(0)=\big(\rho(X)^{-1}\odot\psi_0(X)\big)^\top\Psi(X)^\dagger,\qquad \psi_0(X)\in\mathbb{C}^{N_c\times 1},
\end{equation}
%which is similarly decomposed into real and imaginary parts for initialization.
and we can solve \cref{eqn:ode} efficiently by some ODE solvers. In our numerical experiments later, the solver is fixed as the explicit adaptive Dormand--Prince 5(4) pair (DOPRI5) with relative tolerance $10^{-10}$ and absolute tolerance $10^{-12}$. % unless otherwise specified. 
%The same solver is used in all experiments of \Cref{sec:exp}.
%The $2(M+1)$-dimensional real vector $(C_{\text{re}}, C_{\text{im}})$ is treated as the state variable.
%After each step, we reconstruct $C(t) = C_{\text{re}}(t) + iC_{\text{im}}(t)$ and then $\hat{\psi}(\bm{x},t)=\rho(\bm{x})C(t)\Psi(\bm{x})$.

Recall $\hat{u}(\bm{x},t) = C(t)\Psi(\bm{x})$ and we have fixed for the activation function $\sigma(z):=\tanh(z)$. Since the hidden parameters are fixed, all
 required spatial derivatives are indeed explicit ($\sigma'=1 - \sigma^2,\,\sigma''=-2\sigma(1-\sigma^2)$ applied entry-wise to vectors):
\begin{itemize}
\item The \textit{gradient} of the approximation is
    \begin{equation}\label{eqn:ansatz_x}
    \nabla_{\bm{x}}\hat{u}(\bm{x},t) = C(t)\begin{bmatrix} W \odot \tilde{\sigma}_{\bm{x}}(\bm{x}) \\ \mathbf{0}_{1 \times d} \end{bmatrix} \in \mathbb{C}^{1 \times d},
    \end{equation}
     where %$\mathbf{0}_{1 \times d}$ is a zero row vector, $\mathbf{1}_d \in \mathbb{R}^{d \times 1}$ is a column vector of ones, and
    $
        \tilde{\sigma}_{\bm{x}}(\bm{x}) := \sigma'\!\left(W\bm{x}^\top+b\right)\,\mathbf{1}_d^\top \in \mathbb{R}^{M \times d}$.
    %is the rank-one broadcast of $\sigma'(W\bm{x}^\top+b)\in\mathbb{R}^{M\times 1}$ across $d$ columns.
    \item The \textit{per-dimension second-order derivatives} of the approximation is
     \begin{equation}\label{eqn:ansatz_xx}
         \mathrm{diag}\!\big(\nabla^2_{\bm{x}}\hat{u}\big)(\bm{x},t) := \big(\partial_{x_1}^2\hat u,\ldots,\partial_{x_d}^2\hat u\big)(\bm{x},t)= C(t)\begin{bmatrix} W \odot W \odot \tilde{\sigma}_{\bm{x}\bm{x}}(\bm{x}) \\ \mathbf{0}_{1 \times d} \end{bmatrix} \in \mathbb{C}^{1 \times d},
     \end{equation}
    where $\tilde{\sigma}_{\bm{x}\bm{x}}(\bm{x}) := \sigma''(W\bm{x}^\top+b)\,\mathbf{1}_d^\top \in \mathbb{R}^{M\times d}$. % and $\sigma''= -2\sigma\sigma'$.
    Summing the above further leads to the \textit{Laplacian},
    \begin{align}
    \begin{split}\label{eqn:laplacian}
        \nabla^2 \hat{u}(\bm{x},t) &= \sum_{j=1}^{d} \frac{\partial^2 \hat{u}}{\partial x_j^2} = C(t)\begin{bmatrix} W \odot W \odot \tilde{\sigma}_{\bm{x}\bm{x}}(\bm{x}) \\ \mathbf{0}_{1 \times d} \end{bmatrix} \mathbf{1}_d \in \mathbb{C}.
    \end{split}
    \end{align}
    %where $\tilde{\sigma}_{\bm{x}\bm{x}}(\bm{x}) := \sigma''(W\bm{x}^\top+b)\,\mathbf{1}_d^\top \in \mathbb{R}^{M\times d}$.
   % which is the trace of the Hessian.
\end{itemize}

\subsubsection{SVD-based basis compression}
Next, we apply a singular value decomposition (SVD)-based linear transformation to reduce the stiffness and size of the associated ODEs \cref{eqn:ode}.
%To achieve this, we orthogonalize the basis functions using a truncated SVD.
For the augmented feature matrix $
    \Psi(X) = [\Phi(X), \mathbf{1}_{N_c}]^\top \in \mathbb{R}^{(M+1)\times N_c}$, 
we compute its truncated SVD:
\begin{equation*}
    U_r\Sigma_r V_r^\top \approx \Psi(X),\qquad U_r\in\mathbb{R}^{(M+1)\times r},\ \ \Sigma_r\in\mathbb{R}^{r\times r},\ \ V_r\in\mathbb{R}^{N_c\times r},
\end{equation*}
where $\Sigma_r=\mathrm{diag}(\nu_1,\ldots,\nu_r)$ with $\nu_1\geq\nu_2\geq\cdots\geq\nu_r>0$ the ordered singular values. The truncation rank $r\leq M+1$ is selected by a relative threshold $\epsilon_{\text{SVD}}\in(0,1)$ as
\begin{equation}\label{eq:svd_threshold}
    r := \max\{k:\nu_k/\nu_1>\epsilon_{\text{SVD}}\},
\end{equation}
and we then define the reduced orthogonal basis
\begin{equation}\label{eq:psi_r}
    \Psi_r(X) := U_r^\top \Psi(X) \in \mathbb{R}^{r \times N_c}.
\end{equation}
For any evaluation set $Y\in\mathbb{R}^{N_Y\times d}$, including $X_{\text{test}}$, we can use the same $U_r$  to obtain $\Psi_r(Y)=U_r^\top\Psi(Y)\in\mathbb{R}^{r\times N_Y}$.
The derivative matrices can be transformed accordingly:
\begin{equation}\label{eq:svd_transform}
    \partial_{x_j}\Psi_r(X) := U_r^\top\partial_{x_j}\Psi(X),\quad \nabla^2\Psi_r(X) := U_r^\top\nabla^2\Psi(X),%\quad C(t)\in\mathbb{C}^{1\times r},
\end{equation}
and $(\Psi,\partial_{x_j}\Psi,\nabla^2\Psi)$ in \cref{eqn:ode} can be replaced with $(\Psi_r,\partial_{x_j}\Psi_r,\nabla^2\Psi_r)$, since the transformation $U_r^\top$ commutes with the linear spatial differential operators.
This improves the orthogonality of the basis rows on $X$ and the conditioning of the pseudo-inverse. For numerical experiments later, we will use $\epsilon_{\text{SVD}}=10^{-12}$.

\subsubsection{Structure-preservation} In addition to the enforced decay at the far field, 
we now incorporate two more physical components: the normalization of mass  \cref{eqn:norm} and the preservation of   energy \cref{eqn:energy}. %Mass and the quadratic Hamiltonian are enforced at the discrete level up to floating-point roundoff; the quartic interaction term is monitored as a diagnostic.

{\textbf{Discrete mass normalization.}} 
To enforce discrete mass conservation, we introduce a coefficient-level normalization. %Recall the single-point feature column $\Psi(\bm{x})\in\mathbb{R}^{(M+1)\times 1}$ introduced after \cref{eqn:ansatz}, so that $\hat{\psi}(\bm{x},t)=\rho(\bm{x})C(t)\Psi(\bm{x})$. 
First, discretizing the integral  
$
\lVert \hat{\psi}(\cdot,t)\rVert_{L^2}^2$
 with the same collocation points and weights $\omega_n>0$ consistent with the sampling distribution  yields
\begin{equation}\label{eq:mass_disc}
\mathcal{M}(t) :=\sum_{n=1}^{N_c} \omega_n|\hat{\psi}(\bm{x}_n,t)|^2= \sum_{n=1}^{N_c} \omega_n \rho(\bm{x}_n)^2\lvert C(t)\Psi(\bm{x}_n)\rvert^2 = C(t) \, G \, C(t)^*,
\end{equation}
where the real symmetric Gram matrix $G$ is time-independent and precomputable,
\begin{equation*}
G=(G_{k,j})\in \mathbb{R}^{(M+1) \times (M+1)},\quad G_{kj} = \sum_{n=1}^{N_c} \omega_n \rho(\bm{x}_n)^2 \, \Psi_k(\bm{x}_n) \, \Psi_j(\bm{x}_n).
\end{equation*}
Equivalently, $G = \Psi(X)\,W_\rho\,\Psi(X)^\top$, by introducing 
\begin{equation}\label{eq:Wrho}
W_\rho := \mathrm{diag}\!\left(\omega_1\rho(\bm{x}_1)^2,\ldots,\omega_{N_c}\rho(\bm{x}_{N_c})^2\right) \in \mathbb{R}^{N_c\times N_c}.
\end{equation}
%we write $G = \Psi(X)\,W_\rho\,\Psi(X)^\top$, which we use later in the well-posedness analysis.
Suppose that we are solving ${C}(t)$ by \cref{eqn:ode} in each time interval $t\in[t_n,t_{n+1}]$ of the time step $\Delta t>0$ for $n\geq0$. With the (numerically) obtained $C(t_{n+1})$, we apply a projection: 
\begin{equation}
\label{eq:post_norm}
C(t_{n+1}) \to \frac{C(t_{n+1})}{\sqrt{C(t_{n+1}) \, G \, C(t_{n+1})^*}},
\end{equation}
and the approximated wave function at each time grid %is modified to
%\begin{equation}\label{eq:norm_ansatz}
$\hat{\psi}(\bm{x},t_{n+1}) = \rho(\bm{x})C(t_{n+1})\Psi(\bm{x})$
%\end{equation}
then gives $\mathcal{M}(t_{n+1})=1$.
If we begin with the projection on
the initial coefficient \cref{eqn:C0},
we can thus keep $\mathcal{M}(t_{n})=1$ for all $n\geq0$. This step-wise projection enforces discrete mass conservation with respect to the chosen quadrature rule, at a marginal overhead of $\mathcal{O}(M^2 + M N_c)$.
%For the unnormalized evolution, we set $\mathcal{Z}(t) := 1/\sqrt{\tilde{C}(t) G \tilde{C}(t)^*}$, a real positive scalar. We then write $\hat{\psi}(\bm{x},t) = \rho(\bm{x})\mathcal{Z}(t) \, \tilde{C}(t) \, \Psi(\bm{x})$.
%Time differentiation then yields
%\begin{equation*}
%\partial_t \hat{\psi}(\bm{x},t) = \rho(\bm{x})\dot{\mathcal{Z}}(t) \, \tilde{C}(t) \, \Psi(\bm{x}) + \rho(\bm{x})\mathcal{Z}(t) \, \tilde{C}_t(t) \, \Psi(\bm{x}).
%\end{equation*}

%Since $\mathcal{Z}(t)>0$ is real, $\dot{\mathcal{Z}}(t)$ represents an amplitude-rescaling term induced by normalization. 
%Rather than evolving this term as part of an enlarged constrained system, we use a projection method.
%The unnormalized coefficients $\tilde{C}(t)$ are evolved by \cref{eqn:ode} and then projected onto the discrete mass manifold $\{C:CGC^*=1\}$ via the rescaling

%Furthermore, it preserves the causal ODE structure from \Cref{sec:separat} and remains compatible with our framework, since $G$ encodes the spatial sampling.

%\begin{remark}[Projection effect of mass normalization]\label{rem:splitting}
%Let $\mathcal{S}_{\Delta t}^{\mathrm{RK}}$ denote one numerical RK step for \cref{eqn:ode_system} and let $\Pi_G:C\mapsto C/\sqrt{CGC^*}$. The map $\Pi_G\circ\mathcal{S}_{\Delta t}^{\mathrm{RK}}$ should be viewed as a projection method rather than as an exactly symplectic integrator. It enforces $CGC^*=1$ exactly whenever the denominator is nonzero, and constitutes the first step of the two-stage mass--energy projection introduced in \cref{eq:energy_projection}.
%\end{remark}

{\textbf{Discrete energy conservation.}} 
Optionally, we can construct a tangent projection on the discrete mass sphere that enforces the  discrete Hamiltonian conservation jointly with the mass constraint. To this end, we are
%\begin{equation}\label{eq:energy_cont}
%E(t) = \int_{\mathbb{R}^d} \left[ \frac{1}{2} \lvert \nabla \hat{\psi}(\bm{x},t) \rvert^2 + V_d(\bm{x}) \lvert \hat{\psi}(\bm{x},t) \rvert^2 + \frac{\beta}{2} \lvert \hat{\psi}(\bm{x},t) \rvert^4 \right] d\bm{x},
%\end{equation}
within the space-time separated framework of \Cref{sec:separat}, and we work with the SVD-compressed basis $\Psi_r$ of \cref{eq:psi_r}, so that $C(t)\in\mathbb{C}^{1\times r}$ with $r\leq M+1$. %The uncompressed case is recovered by setting $r=M+1$ and $\Psi_r=\Psi$.
We discretize \cref{eqn:energy} at the collocation points $X = \{\bm{x}_n\}_{n=1}^{N_c}$ with quadrature weights $\{\omega_n\}_{n=1}^{N_c}$ from \cref{eqn:ode}. The discrete energy then decomposes into a quadratic part and a quartic interaction part:
\begin{equation}\label{eq:energy_disc}
\mathcal{E}(C(t))=\underbrace{C\,H\,C^*}_{\mathcal{H}_2(C)}\;+\;\underbrace{\tfrac{\beta}{2}\sum_{n=1}^{N_c}\omega_n\,|\hat{\psi}(\bm{x}_n,t)|^4}_{E_{\mathrm{int}}(C)},\quad 
\hat{\psi}(\bm{x}_n,t)=\rho(\bm{x}_n)C(t)\Psi_r(\bm{x}_n),
\end{equation}
where $H\in\mathbb{C}^{r\times r}$ is the Hermitian matrix encoding kinetic and potential contributions,
\begin{equation}\label{eq:Hmat}
\begin{split}
H_{ij}=&\tfrac12\!\sum_{n}\omega_n\,\overline{\nabla_{\bm x}[\rho(\bm{x}_n)\,\Psi_{r,i}(\bm{x}_n)]}\!\cdot\!\nabla_{\bm x}[\rho(\bm{x}_n)\,\Psi_{r,j}(\bm{x}_n)]\\
&+\sum_n\omega_nV_d(\bm{x}_n)\rho(\bm{x}_n)^2\,\overline{\Psi_{r,i}(\bm{x}_n)}\,\Psi_{r,j}(\bm{x}_n),\end{split}
\end{equation}
with $\Psi_{r,i}$ denoting the $i$-th row of $\Psi_r$. 
Let $C^0$ be the initial coefficient after the mass projection \cref{eq:post_norm}, and the target discrete energy is
\begin{equation}\label{eq:energy_target}
E_0:=\mathcal{E}(C^0)=\mathcal{H}_2(C^0)+E_{\mathrm{int}}(C^0).
\end{equation}
Given (numerically obtained) ${C}(t_{n+1})$ from \eqref{eqn:ode}, we construct $C^{n+1}$ satisfying both $C^{n+1}G(C^{n+1})^*=1$ and $\mathcal{E}(C^{n+1})=E_0$ via a
{\textit{two-step closed-form projection}}. 
%We enforce $\mathcal{E}(C)=E_0$ as a hard constraint jointly with $CGC^*=1$ (compressed Gram matrix $G=\Psi_r(X)W_\rho\Psi_r(X)^\top$)  

%{Linear case ($\beta=0$)}: explicit computation.

\emph{Step 1 (mass).} Project for mass as described in \eqref{eq:post_norm}:
\begin{equation}\label{eq:step1_mass}
C^{(1)}:={C}(t_{n+1})\big/\sqrt{{C}(t_{n+1})\,G\,({C}(t_{n+1}))^*}.
\end{equation}

\emph{Step 2 (energy: tangential projection on mass sphere).} With $v:=G^{-1}H(C^{(1)})^{*}\in\mathbb{C}^{r\times 1}$, define
\begin{equation*}%\label{eq:energy_moments}
\begin{split}
&E_1\!:=\!C^{(1)}H(C^{(1)})^{*},\;\;
E_2\!:=\!v^{*}Gv,\;\;
\sigma_E\!:=\!\sqrt{E_2-E_1^{2}},\\
&\hat g\!:=\!\sigma_E^{-1}(v^{*}-E_1 C^{(1)}),\;\;
E_g\!:=\!\hat g\,H\,\hat g^{*}.
\end{split}
\end{equation*}
Since $H$ is Hermitian and $G$ is real symmetric positive definite, $E_1,E_g\in\mathbb{R}$ and $E_2\geq 0$. Moreover, $E_2-E_1^{2}=v^{*}Gv-(C^{(1)}Gv)(v^{*}GC^{(1)*})\geq 0$ by the Cauchy--Schwarz inequality applied to $C^{(1)}$ and $v$ (using $C^{(1)}G(C^{(1)})^{*}=1$), so $\sigma_E\geq 0$ is real and well-defined. If $\sigma_E=0$, then $v^{*}=E_1 C^{(1)}$ and $\mathcal{H}_2$ is locally stationary at $C^{(1)}$ on the mass sphere, so we set $C^{n+1}:=C^{(1)}$. Otherwise, $\hat g$ is well-defined and $G$-orthonormal to $C^{(1)}$, i.e., $\hat g G\hat g^{*}=1$ and $C^{(1)}G\hat g^{*}=0$. 
Along the geodesic $C_\theta=\cos\theta\,C^{(1)}+\sin\theta\,\hat g$, the mass is preserved, and
$
\mathcal{H}_2(C_\theta)=E_1\cos^{2}\theta+2\sigma_E\sin\theta\cos\theta+E_g\sin^{2}\theta.$ 
 With $a_n:=\rho(\bm{x}_n)C^{(1)}\Psi_r(\bm{x}_n)$ and $b_n:=\rho(\bm{x}_n)\hat g\,\Psi_r(\bm{x}_n)$, define $f_n(\theta):=\cos\theta\,a_n+\sin\theta\,b_n$ so that $\hat{\psi}(\bm{x}_n;C_\theta)=f_n(\theta)$. The total energy mismatch
\begin{equation}\label{eq:Phi_def}
\mathcal{R}(\theta):=\mathcal{H}_2(C_\theta)+\tfrac{\beta}{2}\sum_{n=1}^{N_c}\omega_n\bigl|f_n(\theta)\bigr|^{4}-E_0
\end{equation}
is a real trigonometric polynomial in $\theta$ of degree four, with closed-form derivative
\begin{equation*}%\label{eq:Phi_prime}
\mathcal{R}'(\theta) = (E_g - E_1)\sin 2\theta + 2\sigma_E\cos 2\theta + 2\beta\sum_{n=1}^{N_c}\omega_n\,|f_n(\theta)|^{2}\,\mathrm{Re}\bigl[\overline{f_n(\theta)}\,f_n'(\theta)\bigr],
\end{equation*}
where $f_n'(\theta)=-\sin\theta\,a_n+\cos\theta\,b_n$. %For an order-$p$ RK scheme, $|\mathcal{R}(0)|=\mathcal{O}(\Delta t^{p+1})$, and if $\mathcal{R}'(0)\neq 0$, a simple root $\theta^{\star}=\mathcal{O}(\Delta t^{p+1})$ 
The root $\theta^{\star}$ of \eqref{eq:Phi_def} can be found by Newton's iteration
$
\theta_{k+1}=\theta_k-\mathcal{R}(\theta_k)/\mathcal{R}'(\theta_k),\  k=0,1,\dots,$
warm-started at $\theta_0=0$ and terminated when $|\mathcal{R}(\theta_k)|\leq\mathrm{tol}_E$. %The warm start lies in the basin of quadratic convergence, 
In practice, two or three iterations typically suffice. 
Then  $C^{n+1}=\cos\theta^{\star}\,C^{(1)}+\sin\theta^{\star}\,\hat g$ satisfies $C^{n+1}G(C^{n+1})^{*}=1$ and $\mathcal{E}(C^{n+1})=E_0$ up to tolerance $\mathrm{tol}_E$ (set as $10^{-12}$ in experiments). %Relative to the closed form \cref{eq:theta_sol}, 
Each step requires an inner solver with per-iteration cost $\mathcal{O}(N_c)$. %On the rare events of $\sigma_E=0$, $|\mathcal{R}'(\theta_k)|$ below a safeguard threshold, or non-convergence within $K_{\max}$ iterations, the algorithm falls back to $C^{n+1}:=C^{(1)}$. As the closed-form quadratic projection with adaptive $E_{\mathrm{int}}$ control suffices (\Cref{fig:ablation_conservation}), this variant is not pursued.

\begin{figure}[t!]
    \centering
    \includegraphics[width=0.88\textwidth]{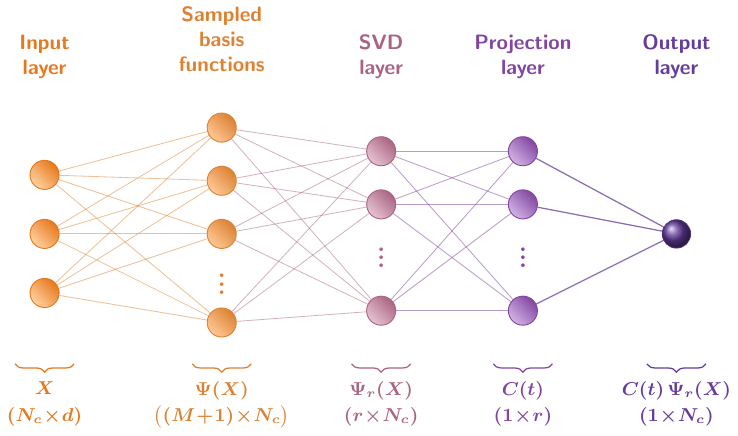}
    \vspace{-0.5em}
    \caption{\textbf{Architecture of SD-FSNN.} Collocation points $X$ pass through the frozen augmented basis $\Psi(X)=[\Phi(X),\mathbf{1}_{N_c}]^\top$, the SVD compression $\Psi_r(X)=U_r^\top\Psi(X)$, and the mass--energy projection on $C(t)$ to produce the output $C(t)\,\Psi_r(X)$. $C(t)$ evolves in time, with spatial parameters frozen.}
    \label{fig:sdfsnn-net}
\end{figure}
\begin{remark}[Linear case: $\beta=0$]
Let $A=\tfrac{E_1-E_g}{2},\;\;B=\sigma_E,\;\;D=E_0-\tfrac{E_1+E_g}{2}$, and $\mathcal{H}_2(C_\theta)=E_0$ gives 
$A\cos 2\theta+B\sin 2\theta=D.$
This equation is solvable iff $|D|\leq\sqrt{A^{2}+B^{2}}$ which is equivalent to $E_0\in[\,(E_1+E_g)/2 - \sqrt{A^{2}+B^{2}},\,(E_1+E_g)/2 + \sqrt{A^{2}+B^{2}}\,]$. For a convergent integrator, the condition can be satisfied for sufficiently small $\Delta t$. Then, the smallest-magnitude root is given in closed form by
\begin{equation*}%\label{eq:theta_sol}
\theta^{\star}=\tfrac12\!\left[\operatorname{atan2}(B,A)-\arccos\!\big(D/\sqrt{A^{2}+B^{2}}\big)\right]\!,
\end{equation*}
with $\operatorname{atan2}:\mathbb{R}^{2}\setminus\{(0,0)\}\to(-\pi,\pi]$ the four-quadrant arctangent and the principal branch $\arccos:[-1,1]\to[0,\pi]$. %The projected coefficient is then
%\begin{equation}\label{eq:energy_projection}
%$C^{n+1}=\cos\theta^{\star}\,C^{(1)}+\sin\theta^{\star}\,\hat g$
%which satisfies $C^{n+1}G(C^{n+1})^{*}=1,\,\mathcal{E}(C^{n+1})=\mathcal{H}_2(C^{n+1})=E_0$ by construction. 
The total cost per time level is $\mathcal{O}(M^{2}+M N_c)$, identical to that of mass-only projection \cref{eq:post_norm}. %The feasibility condition $|D|\leq\sqrt{A^{2}+B^{2}}$ in \cref{eq:theta_eq} is equivalent to $E_0\in[\,(E_1+E_g)/2 - \sqrt{A^{2}+B^{2}},\,(E_1+E_g)/2 + \sqrt{A^{2}+B^{2}}\,]$. For a convergent scheme, the condition can be satisfied for sufficiently small $\Delta t$.
\end{remark}

The resulting basic architecture of SD-FSNN is illustrated in \Cref{fig:sdfsnn-net}, consisting of five conceptual layers. Next, we add the stochastic-dimension sampling on it. 

\subsection{Stochastic-dimension operator sampling}
\label{sec:sdfsnn_weights}
Inspired by stochastic coordinate descent methods \cite{nesterov2012efficiency}, we now develop a stochastic-dimension sampling for differential operators in SD-FSNN to enable efficient high-dimensional evaluation.
% As another central new ingredient for efficient evaluation under high dimensionality, we now develop a stochastic-dimension sampling for differential operators in SD-FSNN.
%The construction proceeds in three steps: (i) exposing the dimension-wise additive structure of the de-enveloped spatial operator, (ii) defining an unbiased dimension-batch sampler that amortizes this sum, and (iii) plugging the sampler into the projected ODE for the coefficients $C(t)$.
%Recall that $\hat{\psi}(\bm{x},t)=\rho(\bm{x})\hat u(\bm{x},t)$ with $\hat u(\bm{x},t)=C(t)\Psi(\bm{x})$.

Firstly, we reveal the dimension-wise additive structure of the de-enveloped spatial operator. 
Using the envelope-induced coefficients introduced in \cref{eqn:envelope_coeffs}, the spatial part of the de-enveloped equation admits the dimension-wise additive decomposition
\begin{equation}
\label{eq:lap_additive}
\mathcal{A}_{\rho}\hat u
=\sum_{j=1}^{d}\mathcal{A}_{\rho,j}\hat u,\qquad
\mathcal{A}_{\rho,j}\hat u
:=\frac12\frac{\partial^2\hat u}{\partial x_j^2}
+a_j\frac{\partial\hat u}{\partial x_j}
+\frac12 q_j \hat u,
\end{equation}
with $a_j,q_j$ given in \eqref{eqn:envelope_coeffs}. 
%$q=\sum_{j=1}^d q_j$ recovers the Laplacian-ratio coefficient already used in \cref{eqn:ode}.
For our neural basis  $\phi_m(\bm{x})=\sigma(w_m \bm{x}^\top+b_m)$ with $w_m=(w_{m,1},\ldots,w_{m,d})$, the  derivatives
$\partial_{x_j}\phi_m(\bm{x})=w_{m,j}\,\sigma'(w_m\bm{x}^\top+b_m)$ and
$\partial^2_{x_j}\phi_m(\bm{x})=w_{m,j}^2\,\sigma''(w_m\bm{x}^\top+b_m)$
show that the $j$-th additive contribution depends only on the $j$-th component of the frozen weight together with the corresponding envelope derivatives in direction $j$.
This per-coordinate separability is precisely what could make dimension-wise sampling effective.

Thus, we propose an \textit{unbiased dimension-batch sampler}. 
Building on the additive structure \eqref{eq:lap_additive}, we amortize the dimension-wise sum by sampling only a small subset of coordinates at each evaluation from $t_n\to t_{n+1}$.
Let $J\subset\{1,\dots,d\}$ be a random index set with $|J|=s \ll d$, drawn uniformly so that $\mathbb{P}(j\in J)=s/d$ for every $j$.
The stochastic-dimension spatial-operator sampler is then defined as
\begin{equation}
\label{eq:sdfsnn_laplacian}
\widetilde{\mathcal{A}}_{\rho,J}\hat u(\bm{x},t)
:= \frac{d}{s}\sum_{j\in J}\mathcal{A}_{\rho,j}\hat u(\bm{x},t),\quad t_n\leq t\leq t_{n+1}.
\end{equation}
By linearity and symmetry of uniform sampling, this estimator is unbiased, stated rigorously in the next subsection.  
%\begin{equation}\label{eq:sdfsnn_unbiased_lap}
%$\mathbb{E}_{J}[\widetilde{\mathcal{A}}_{\rho,J}\hat u(\bm{x},t)]
%= \mathcal{A}_{\rho}\hat u(\bm{x},t).$
%\end{equation}
A simple non-uniform variant retains the same property: for $J=\{j_1,\ldots,j_s\}$ with $j_1,\ldots,j_s\stackrel{\mathrm{i.i.d.}}{\sim}p(\cdot)$, $p_j>0$ and $\sum_{j=1}^d p_j=1$, the inverse-probability estimator
%\begin{equation}\label{eq:sdfsnn_lap_nonuniform}
$\frac{1}{s}\sum_{\ell=1}^{s}\frac{1}{p_{j_\ell}}
\mathcal{A}_{\rho,j_\ell}\hat u(\bm{x},t)$
is also unbiased. 
Crucially, each sampled coordinate $j\in J$ activates only its corresponding weight components $\{w_{m,j}\}_{m=1}^{M}$ and the envelope derivatives in direction $j$, so \cref{eq:sdfsnn_laplacian}  reduces the spatial-operator evaluation cost from $\mathcal{O}(Md)$ to $\mathcal{O}(Ms)$.

%\paragraph{\textbf{SD-FSNN residual and induced ODE}}
We now substitute the sampler \eqref{eq:sdfsnn_laplacian} into the  ODEs \cref{eqn:ode} for each time level $t_n\leq t\leq t_{n+1}$. Note that the deterministic right-hand side reads
\begin{align*}
R_\rho(X,C(t))
= i\big[
\mathcal{A}_{\rho}\hat u(X,t)
- V_d(X)^\top\odot U(X,t)
- \beta(\rho(X)^2)^\top\odot |U(X,t)|^2\odot U(X,t)
\big],%\label{eqn:R_exact}
\end{align*}
%with projected vector field
%\begin{equation}\label{eqn:F_exact}
%F_\rho(C(t)) := R_\rho(X,C(t))\Psi(X)^\dagger,
%\end{equation}
then its stochastic-dimension counterpart is obtained by replacing $\mathcal{A}_{\rho}$ with $\widetilde{\mathcal{A}}_{\rho,J}$:
\begin{align}
\label{eq:sdfsnn_ode}
\dot{C}(t)
&:= \widetilde{R}_{\rho,J}(X,C(t))\Psi(X)^\dagger,\quad t_n\leq t\leq_{t_{n+1}},\\
\widetilde{R}_{\rho,J}(X,C(t))
&:= i\Big[
\widetilde{\mathcal{A}}_{\rho,J}\hat u(X,t)
- V_d(X)^\top \odot U(X,t) \nonumber\\
&\qquad
- \beta(\rho(X)^2)^\top\odot |U(X,t)|^2\odot U(X,t)
\Big], \nonumber
\end{align}
and the sampled dimensions will be renewed at the next time level. 
For each fixed $(W,b,X,C(t))$ and an independent draw of $J$, the stochastic right-hand side is conditionally unbiased
$
\mathbb{E}_{J}\!\left[\widetilde{R}_{\rho,J}(X,C(t))\,\middle|\,W,b,X,C(t)\right] = R_\rho(X,C(t))$ (will be proved below), 
so its variance can be reduced by enlarging $|J|$ or by averaging $K$ independent dimension batches per evaluation, which divides the conditional variance by $K$.
%In \Cref{alg:sdfsnn} we hold $J$ fixed across all RK stages of a single time step; this removes inter-stage sampling fluctuations and prevents the local truncation error estimate of DOPRI5 from being polluted by stage-dependent stochastic biases.

%\subsection{Training procedures and implementation}
%Having established well-posedness of the semi-discrete system, we outline the implementation of SD-FSNN, which requires no loss function or iterative backpropagation. 
The complete SD-FSNN procedure, including sampling, projection and time evolution, is summarized in \Cref{alg:sdfsnn}. Within each time level in our implementation, the index set $J$ is fixed across all DOPRI5 stages, which removes inter-stage sampling fluctuations and prevents the local truncation error from being polluted by stage-dependent stochastic biases.

\begin{algorithm}[h!]
\caption{Full computational procedure of SD-FSNN}
\label{alg:sdfsnn}
\textbf{Input}: Initial condition $\psi_0$, final time $T>0$, test grid points $X_{\text{test}}$, evaluation times $T_{\text{test}}\subset[0,T]$, dimension batch size $s \ll d$\\
\textbf{Output}: Predicted PDE solution $\hat{\psi}(X_{\text{test}}, T_{\text{test}})$\\
\textbf{Parameters}: $N_c$, $M$, $\epsilon_{\text{SVD}}$, $\rho$%, tolerances $(\mathrm{rtol},\mathrm{atol})$
\begin{algorithmic}[1]
\setlength{\itemsep}{0pt}
\setlength{\parsep}{0pt}
\STATE Sample $N_c$ collocation points $X\in\mathbb{R}^{N_c\times d}$.
\STATE Initialize frozen parameters $\{W,b\}$: for each $1\leq m\leq M$, sample a pair $(\bm{x}_{m,1},\bm{x}_{m,2})$ via the data-driven density and set $(w_m,b_m)$ by \cref{eqn:sample_weight}.
\STATE Compute augmented feature matrix: $\Psi(X) = [\Phi(X), \mathbf{1}_{N_c}]^\top \in \mathbb{R}^{(M+1) \times N_c}$.
\STATE Assemble $W_\rho$ as in \cref{eq:Wrho}.
\STATE Compute truncated SVD: $U_r \Sigma_r V_r^\top \approx \Psi(X)$ using threshold $\epsilon_{\text{SVD}}$.
\STATE Obtain reduced orthogonal bases: $\Psi_r(X) = U_r^\top \Psi(X) \in \mathbb{R}^{r \times N_c}$; set $G= \Psi_r(X)W_\rho\Psi_r(X)^\top$.
\STATE Initialize coefficients via least-squares: $C(0) = \big(\rho(X)^{-1} \odot \psi_0(X)\big)^\top \Psi_r(X)^\dagger$.
\STATE Apply initial mass normalization: $C^0=  C(0) / \sqrt{C(0) G C(0)^*}$.
%\STATE Assemble $H$ via \cref{eq:Hmat} and set the target quadratic energy $E_0^{\mathrm{quad}} := \mathcal{H}_2(C(0))$ via \cref{eq:energy_target}.
\FOR{each DOPRI5 step $t_n \to t_{n+1}$ with $0\leq t_n\le T$}
    \STATE Uniformly sample a random dimension subset $J \subset \{1, \dots, d\}$ with $|J| = s$, held fixed across all RK stages of this step.
   % \STATE Compute $\widetilde{\mathcal{A}}_{\rho,J}\hat u(X,t_n)$ via \cref{eq:sdfsnn_laplacian}.
    \STATE Evaluate $\widetilde{R}_{\rho,J}(X, C)$ for ODEs \cref{eq:sdfsnn_ode}.
    \STATE Obtain ${C}(t_{n+1})$ by one DOPRI5 step applied to \cref{eq:sdfsnn_ode}.
    \STATE Projection: mass-project via \eqref{eq:step1_mass}; energy-projection (optinal) for $C^{n+1}$.  %compute the scalars in \cref{eq:energy_moments}; 
   % if $\sigma_E=0$ or \cref{eq:theta_eq} is infeasible, set $C^{n+1}=C^{(1)}$. %otherwise form the tangent $\hat g$, solve $\theta^\star$ from \cref{eq:theta_sol}, and set $C(t+\Delta t)$ via \cref{eq:energy_projection}.
\ENDFOR
\STATE \textbf{Return}: for each $t_n\in T_{\text{test}}$, $\hat{\psi}(X_{\text{test}}, t_n) = \rho(X_{\text{test}})^\top \odot \big(C^n\,\Psi_r(X_{\text{test}})\big)$.
\end{algorithmic}
\end{algorithm}

%\paragraph{\textbf{Conditional unbiasedness of SD-FSNN}}

\subsection{Numerical analysis}
For the proposed SD-FSNN method, we provide some theoretical analysis results, supporting the unbiasedness of the dimension-batch sampler and the well-posedness of the coefficient dynamics.

\subsubsection{Conditional unbiasedness and variance}
We begin with the bias of the SD-FSNN. The sampler $\widetilde{\mathcal{A}}_{\rho,J}$ is built so that, conditioned on the current state, its expectation over $J$ recovers the deterministic operator $\mathcal{A}_{\rho}$. 
This property propagates through the residual $R_\rho$ in \cref{eqn:ode} and through the projected vector field
\begin{equation}\label{eqn:F_exact}
F_\rho(C(t)) := R_\rho(X,C(t))\Psi(X)^\dagger,
\qquad
\widetilde{F}_{\rho,J}(C(t)) := \widetilde{R}_{\rho,J}(X,C(t))\Psi(X)^\dagger,
\end{equation}
which is the right-hand side of \cref{eqn:ode} (resp.\ of its stochastic counterpart \cref{eq:sdfsnn_ode}) viewed as a map of the coefficient vector alone.

\begin{theorem}[Conditional unbiasedness of SD-FSNN]
\label{thm:sdfsnn_unbiased_gpe}
At $t_n$ of \Cref{alg:sdfsnn}, let $\mathcal{G}_n$ denote the $\sigma$-algebra generated by the frozen parameters $(W,b)$, the collocation set $X$, the coefficient state $C^n$ and the dimension batches $\{J_k\}_{k<n}$ from previous steps. Assume that the current dimension batch $J=J_n$ is sampled uniformly without replacement from $\{1,\dots,d\}$ with $|J|=s$, independently of $\mathcal{G}_n$.
Then, the SD-FSNN stochastic spatial operator satisfies
\begin{equation}\label{eqn:unbiased_lap}
\mathbb{E}\!\left[\widetilde{\mathcal{A}}_{\rho,J}\hat u(\bm{x},t_n)\,\middle|\,\mathcal{G}_n\right]
=\mathcal{A}_{\rho}\hat u(\bm{x},t_n),\qquad \forall\,\bm{x}\in\mathbb{R}^{1\times d}.
\end{equation}
Consequently, with $F_\rho$ and $\widetilde{F}_{\rho,J}$ as in \cref{eqn:F_exact},
\begin{equation}\label{eqn:unbiased_R_F}
\mathbb{E}\!\left[\widetilde{R}_{\rho,J}(X,C^n)\,\middle|\,\mathcal{G}_n\right] = R_\rho(X,C^n),
\qquad
\mathbb{E}\!\left[\widetilde{F}_{\rho,J}(C^n)\,\middle|\,\mathcal{G}_n\right] = F_\rho(C^n).
\end{equation}
%The identities also hold componentwise for the coupled real-valued system obtained by taking $\mathrm{Re}(\cdot)$ and $\mathrm{Im}(\cdot)$.
\end{theorem}

\begin{proof}
Condition on $\mathcal{G}_n$ and set $f_j:=\mathcal{A}_{\rho,j}\hat u(\bm{x},t_n)$, $\delta_j:=\mathbf{1}_{\{j\in J\}}$. By the assumed sampling design, $\mathbb{E}[\delta_j\mid\mathcal{G}_n]=s/d$ for every $j$. Hence,
\[
\mathbb{E}\!\left[\widetilde{\mathcal{A}}_{\rho,J}\hat u(\bm{x},t_n)\,\middle|\,\mathcal{G}_n\right]
=\tfrac{d}{s}\sum_{j=1}^{d}\mathbb{E}[\delta_j\mid\mathcal{G}_n]\,f_j
=\sum_{j=1}^{d} f_j
=\mathcal{A}_{\rho}\hat u(\bm{x},t_n),
\]
which proves \cref{eqn:unbiased_lap}.
For \cref{eqn:unbiased_R_F}, note that the potential term $V_d(X)^{\top}\odot U(X,t_n)$ and the cubic term $\beta(\rho(X)^{2})^{\top}\odot|U(X,t_n)|^{2}\odot U(X,t_n)$ are $\mathcal{G}_n$-measurable, hence constant under $\mathbb{E}[\,\cdot\mid\mathcal{G}_n]$. Only $\mathcal{A}_{\rho}\hat u(X,t_n)$ is replaced by $\widetilde{\mathcal{A}}_{\rho,J}\hat u(X,t_n)$, so conditional linearity together with \cref{eqn:unbiased_lap} applied row-wise yields $\mathbb{E}[\widetilde{R}_{\rho,J}(X,C^n)\mid\mathcal{G}_n]=R_\rho(X,C^n)$. Since $\Psi(X)^{\dagger}$ is deterministic given $\mathcal{G}_n$, right-multiplication preserves the identity, giving the second equality in \cref{eqn:unbiased_R_F}. %The real/imaginary statements follow because $\mathrm{Re}$ and $\mathrm{Im}$ are bounded linear functionals.
\end{proof}

\begin{proposition}
[Variance of dimension-batch sampler]
\label{lem:sdfsnn_variance}
Under assumptions of \Cref{thm:sdfsnn_unbiased_gpe}, define the complex-valued conditional variance by $\mathrm{Var}(Z\mid\mathcal{G}_n):=\mathbb{E}[|Z-\mathbb{E}(Z\mid\mathcal{G}_n)|^2\mid\mathcal{G}_n]$. Let $f_j:=\mathcal{A}_{\rho,j}\hat u(\bm{x},t_n),\ \bar f:=\tfrac{1}{d}\sum_{j=1}^{d} f_j$, then,
\begin{equation}\label{eq:sdfsnn_variance}
\mathrm{Var}\!\left[\widetilde{\mathcal{A}}_{\rho,J}\hat u(\bm{x},t_n)\,\middle|\,\mathcal{G}_n\right]
=\frac{d(d-s)}{s(d-1)}\,\sum_{j=1}^d|f_j-\bar f|^2.
\end{equation}
The prefactor decreases in $s$ and vanishes at $s=d$. 
The dimension scaling is therefore controlled by the empirical coordinate fluctuation $\sum_{j=1}^d|f_j-\bar f|^2$.
\end{proposition}
\begin{proof}
For simple random sampling of size $s$ without replacement, $\mathrm{Var}(\delta_j\mid\mathcal{G}_n)=s(d-s)/d^2$ for every $j$, and $\mathrm{Cov}(\delta_j,\delta_k\mid\mathcal{G}_n)=-s(d-s)/[d^2(d-1)]$ for $j\neq k$. 
Set $g_j:=f_j-\bar f\in\mathbb{C}$, so that $\sum_{j=1}^{d} g_j=0$. Since $\widetilde{\mathcal{A}}_{\rho,J}\hat u(\bm{x},t_n)-\mathcal{A}_{\rho}\hat u(\bm{x},t_n)=(d/s)\sum_{j=1}^d\delta_j g_j$ and $\big|\sum_{j=1}^{d} g_j\big|^2=0$ expands as $\sum_{j} |g_j|^2+\sum_{j\neq k} g_j\overline{g_k}=0$, we have $\sum_{j\neq k} g_j\overline{g_k}=-\sum_{j} |g_j|^2$. Substituting the variance and covariance identities yields
\[
\mathbb{E}\left[\big|\textstyle\sum_{j=1}^{d}\delta_j g_j\big|^{2}\big| \mathcal{G}_n\right]
=\frac{s(d-s)}{d^{2}}\sum_{j=1}^d|g_j|^2-\frac{s(d-s)}{d^{2}(d-1)}\sum_{j\neq k} g_j\overline{g_k}
=\frac{s(d-s)}{d(d-1)}\sum_{j=1}^d|g_j|^2,
\]
which is \cref{eq:sdfsnn_variance} by 
multiplying $(d/s)^2$ on both sides.
\end{proof}

%\subsubsection{Local well-posedness of SD-FSNN}
%\label{sec:well-posedness}
\subsubsection{Well-posedness}
\label{sec:well-posedness}
In addition to the unbiasedness and variance estimates above that control the stochastic forcing locally in time, we now study the 
reduced dynamics on the coefficients. 
%the cubic nonlinearity $|\psi|^{2}\psi$ does not drive the projected coefficients to blow up: we establish local well-posedness of the coefficient ODE and a uniform bound for the mass-projected discrete sequence, with constants that are independent of the random batches $\{J_n\}$.

Recall from \cref{eq:svd_threshold,eq:psi_r} that $r\leq M+1$ is the rank retained by the truncated SVD and that $\Psi_r:=\Psi_r(X)\in\mathbb{R}^{r\times N_c}$ is the resulting orthogonalized feature matrix on the collocation set $X\in\mathbb{R}^{N_c\times d}$. On the matrix,  
the modulus $|\cdot|^{2}$ and the Hadamard product $\odot$ act entry-wise.
The ODEs for the coefficient vector $C(t)\in\mathbb{C}^{1\times r}$ can then be abstracted as the initial value problem
\begin{equation}\label{eq:ode_abstract}
    \frac{dC(t)}{dt}=\mathcal{F}(C(t)):=\mathcal{L}(C(t))+\mathcal{N}(C(t)),\qquad C(0)=C_0,
\end{equation}
where $\mathcal{L},\mathcal{N}:\mathbb{C}^{1\times r}\to\mathbb{C}^{1\times r}$ are the linear and nonlinear parts of $\mathcal{F}$, obtained by restricting the projected vector field $F_\rho$ in \cref{eqn:F_exact} to the SVD-compressed basis $\Psi_r$: %(so that $\Psi(X)^\dagger$ is replaced by $\Psi_r^\dagger$):
\begin{subequations}
\begin{align}
    \mathcal{L}(C) &= i\left[\mathcal{A}_{\rho}(C\Psi_r) - C\Psi_r\,\mathrm{diag}(V_d(X))\right]\Psi_r^{\dagger}, \label{eq:op_L}\\
    \mathcal{N}(C) &= -i\beta\left[(\rho(X)^{2})^{\top}\odot|C\Psi_r|^{2}\odot(C\Psi_r)\right]\Psi_r^{\dagger}. \label{eq:op_N}
    \end{align}
    \end{subequations}
%where $\Psi_r^{\dagger}$ is the Moore--Penrose pseudo-inverse and $\mathcal{A}_\rho$ acts row-wise on the collocation values as in \cref{eq:lap_additive}.

\begin{lemma}[Local existence]\label{lem:lipschitz}
For any $R>0$, the vector field $\mathcal{F}$ in \cref{eq:ode_abstract} is locally Lipschitz on the closed ball $\mathcal{B}_R:=\{C\in\mathbb{C}^{1\times r}:\|C\|_{2}\leq R\}$, with a Lipschitz constant
$
L_R = \|\mathcal{L}\|_{\mathrm{op}}+3|\beta|\,\|\rho(X)^{2}\|_{\infty}\,\|\Psi_r^{\dagger}\|_{\mathrm{op}}\,\|\Psi_r\|_{\mathrm{op}}^{3}\,R^{2}<\infty.$
Consequently, \cref{eq:ode_abstract} admits a unique solution $C\in C^{1}([0,T_{\max});\mathbb{C}^{1\times r})$ for some maximal time $T_{\max}>0$.
\end{lemma}
\begin{proof}
Since $X$ is a finite collocation set, all quantities entering $\mathcal{A}_{\rho}$ are bounded. Hence, \cref{eq:op_L} gives $\|\mathcal{L}(C_1)-\mathcal{L}(C_2)\|_{2}\leq\|\mathcal{L}\|_{\mathrm{op}}\|C_1-C_2\|_{2}$ with $\|\mathcal{L}\|_{\mathrm{op}}<\infty$, where $\|\cdot\|_{2}$ is the Euclidean norm and $\|\cdot\|_{\mathrm{op}}$ the induced spectral operator norm. 

For the nonlinear part, write $U_k:=C_k\Psi_r\in\mathbb{C}^{1\times N_c}$, $k\in\{1,2\}$. For any $C_k\in\mathcal{B}_R$, we have
$
\|U_k\|_{\infty}\leq\|U_k\|_{2}\leq\|C_k\|_{2}\,\|\Psi_r\|_{\mathrm{op}}\leq R\,\|\Psi_r\|_{\mathrm{op}}.$ 
Then, the elementary inequality $\big||z_1|^{2}z_1-|z_2|^{2}z_2\big|\leq 3\max(|z_1|,|z_2|)^{2}|z_1-z_2|$, applied entry-wise, gives
\[
\big\||U_1|^{2}\odot U_1-|U_2|^{2}\odot U_2\big\|_{2}\leq 3\max\{\|U_1\|_{\infty},\|U_2\|_{\infty}\}^{2}\|U_1-U_2\|_{2}\leq 3R^{2}\|\Psi_r\|_{\mathrm{op}}^{3}\|C_1-C_2\|_{2}.
\]
Substituting it into \cref{eq:op_N} yields $$\|\mathcal{N}(C_1)-\mathcal{N}(C_2)\|_{2}\leq 3|\beta|\,\|\rho(X)^{2}\|_{\infty}\,\|\Psi_r^{\dagger}\|_{\mathrm{op}}\,\|\Psi_r\|_{\mathrm{op}}^{3}R^{2}\|C_1-C_2\|_{2}.$$
%with $\|\cdot\|_{\infty}$ the max-norm over collocation entries.
Combining the two bounds gives $L_R$. Local existence and uniqueness of $C$ then follow from the Picard--Lindel\"of theorem.
\end{proof}

\begin{theorem}[Uniform boundness]\label{thm:global_wp}
Assume that $\Psi_r$ has the full row rank and $W_\rho$ in \cref{eq:Wrho} is positive definite (equivalently, $\omega_n>0$ for all $n=1,\ldots,N_c$). The sequence $\{C^{n}\}_{n\geq 0}\subset\mathbb{C}^{1\times r}$ produced by \cref{alg:sdfsnn}  %applying, at every time step, either the mass projection \cref{eq:post_norm} alone (equivalently, Step 1 in \cref{eq:step1_mass}) or the two-step mass--energy projection \cref{eq:step1_mass}--\cref{eq:energy_projection} to one RK step of \cref{eq:ode_abstract}, starting from a mass-normalized $C^{0}$. 
%Then $C^{n}G(C^{n})^{*}=1$ for all $n\geq 0$ and
satisfies
\begin{equation}\label{eq:apriori_bound}
    \sup_{n\geq 0}\|C^{n}\|_{2}\;\leq\;\frac{1}{\sqrt{\lambda_{\min}(G)}}\;<\;\infty,
\end{equation}
where $G=\Psi_r W_\rho \Psi_r^{\top}\in\mathbb{R}^{r\times r}$, and so the bound is independent of $\Delta t,\,n$, the order of the solver and the sequence of dimension batches $\{J_n\}$.
\end{theorem}
\begin{proof}
By assumption, $W_\rho$ is positive definite and $\Psi_r$ has full row rank, so $G=\Psi_r W_\rho\Psi_r^{\top}$ is real symmetric positive definite with $\lambda_{\min}(G)>0$. The projection enforces the mass normalization which gives $C^{n}G(C^{n})^{*}=1$ for every $n\geq 0$. The Rayleigh quotient inequality applied to $G$ then yields $\lambda_{\min}(G)\|C^{n}\|_{2}^{2}\leq C^{n}G(C^{n})^{*}=1$, and so
\[
\|C^{n}\|_{2}\;\leq\;\frac{1}{\sqrt{\lambda_{\min}(G)}},\qquad \forall\,n\geq 0,
\]
where the bound depends only on $\Psi_r$ and $W_\rho$.
\end{proof}

\begin{remark}[Stochastic-dimension setting]\label{rem:sd_stability}
The bound \cref{eq:apriori_bound} is path-wise: on the finite collocation set, the stochastic sampler $\widetilde{\mathcal{A}}_{\rho,J}$ is almost surely bounded, and the induced linear part $\widetilde{\mathcal{L}}(C):=i[\widetilde{\mathcal{A}}_{\rho,J}(C\Psi_r)-C\Psi_r\,\mathrm{diag}(V_d(X))]\Psi_r^{\dagger}$ is uniformly bounded along any sample path. Since $G$ is independent of the sampled subset, the same uniform bound holds for any realization of $\{J_n\}$.
\end{remark}

\section{Numerical experiments}
\label{sec:exp}
This section evaluates SD-FSNN on GPEs over unbounded domains, addressing accuracy  and long-time performance. 
To enable long-term numerical tests, we focus on the defocusing GPE. 
For tests in this section, we shall set in \eqref{eqn:GPE}: $\gamma_j=2.0$ in all directions of the harmonic oscillator $V_d$ and a normalized Gaussian initial data,
\begin{equation}\label{eq:psi0_gauss}
\psi_0(x_1,x_2,\ldots,x_d) = \pi^{-d/4}\,\exp\!\left(-\tfrac{1}{2}\sum_{j=1}^{d} x_j^2\right),
\qquad d=1,2,\ldots,8.
\end{equation}
Unless otherwise specified, the envelope $\rho$ for \eqref{ansatz} uses $\alpha=\frac12$.  
The test errors are computed on a uniform space-time grid over a bounded window. 
All experiments\footnote{Code and data are available at \url{https://github.com/liangzhangyong/SD-FSNN}.} use \texttt{float64} precision, with gradient-free methods run on a single CPU with 64 GB RAM and gradient-based methods on a single \texttt{NVIDIA} A100-SXM4-80GB GPU.

\subsection{Accuracy test}
\label{short_gpe}

%which is normalized to $\|\psi_0\|_{L^2(\mathbb{R}^d)}=1$ and decays as $\|\bm{x}\|\to\infty$. 
%The dimension $d$ enters only through the ambient coordinate count and the prefactor $\pi^{-d/4}$, so the same form serves 
We will begin with accuracy tests for $d=1,2,3$ with reference solutions computed by refined fourth-order time-splitting spectral (TSSP) method \cite{bao2005fourth}. The SD-FSNN is used with mass projection only, where the additional energy projection makes little difference on short time computations. %, and then we will perform the higher-dimensional runtime scaling test.
%We first use 1D GPEs to assess the accuracy of SD-FSNN.
%The Gaussian-weighted ansatz enforces exponential decay at infinity and avoids boundary artifacts induced by domain truncation.
% \vspace{-2em}
\begin{figure}[h!]
\centering
\includegraphics[width=\textwidth]{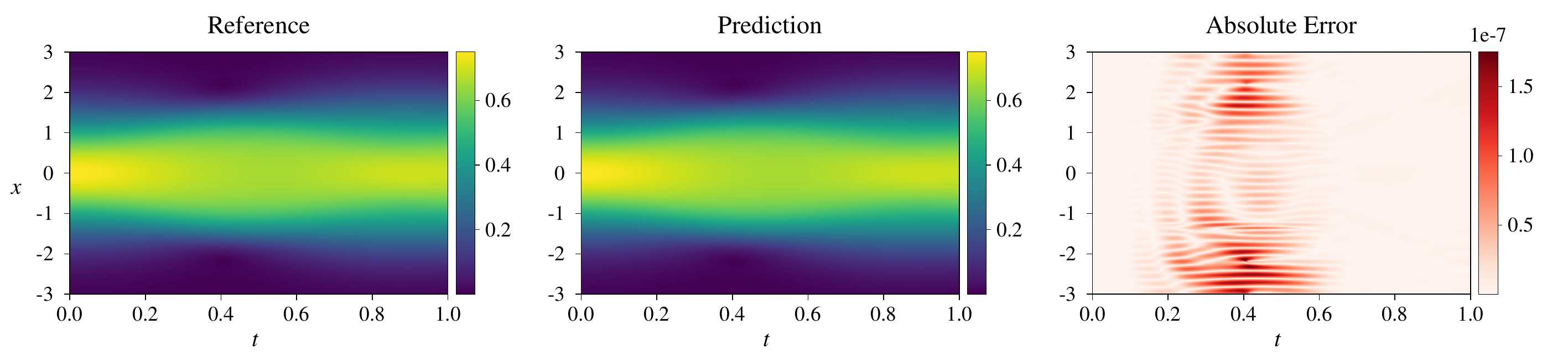}
\vspace{-2em}
\caption{Reference solution, SD-FSNN prediction, and pointwise absolute error in $|\psi|$ for the 1D GPE with $\beta=10$ on $t\in[0,1]$.}
\label{fig:gpe1d_beta10}
\end{figure}
\begin{figure}[h!]
\centering
\includegraphics[width=\textwidth]{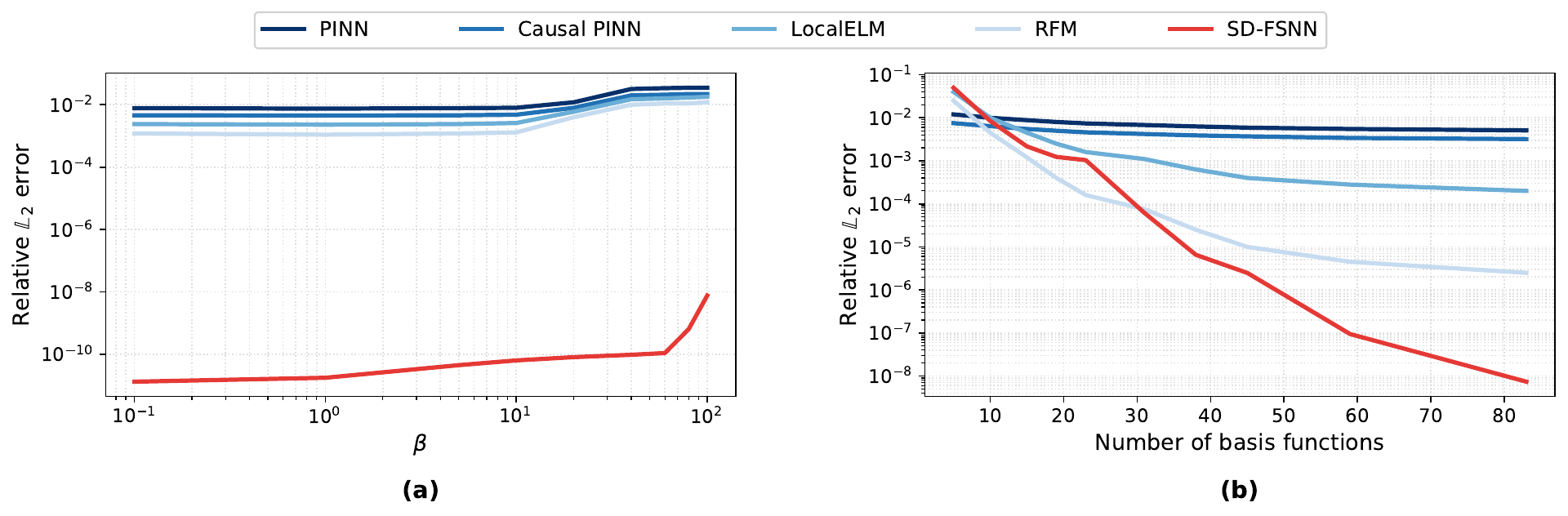}
\vspace{-2em}
\caption{Accuracy comparison between SD-FSNN and the baseline methods on 1D GPE benchmarks: (a) Relative $L^2$ error versus  $\beta$; (b) Relative $L^2$ error versus the number of basis functions.}
\label{fig:sdfsnn_errors}
\end{figure}

\Cref{fig:gpe1d_beta10} presents the space-time evolution of $|\psi|$ for the 1D GPE with $\beta=10$.
The SD-FSNN prediction agrees with the reference solution %and reproduces the localized norm profile near $x=0$ 
over $t \in [0,1]$, where  
the absolute error remains at the level of $10^{-7}$. %, with slightly larger values near the condensate center at later times.
These results indicate a valid approximation of GPE in the unbounded domain at high accuracy. Moreover, we compare the accuracy of SD-FSNN with PINN, Causal PINN, LocalELM, and RFM. 
\Cref{fig:sdfsnn_errors}(a) reports the relative $L^2$ error as a function of $\beta \in [10^{-1},10^{2}]$.
For $\beta \leq 10$, the baselines have errors in the range $10^{-3}$--$10^{-1}$.
Their errors increase for stronger interactions. SD-FSNN remains several orders of magnitude more accurate and stays below $10^{-8}$ at $\beta=100$. %This behavior is consistent with the Gaussian envelope and the closed-form two-step mass--energy projection used in SD-FSNN. 
\Cref{fig:sdfsnn_errors}(b) shows convergence with respect to the number of basis functions.
PINN and Causal PINN show slow convergence, while LocalELM and RFM exhibit early saturation.
By concentrating frozen features near the active region, SD-FSNN reaches errors of order $10^{-8}$ with fewer basis functions than the compared methods.

\begin{figure}[t!]
\centering
\includegraphics[width=\textwidth]{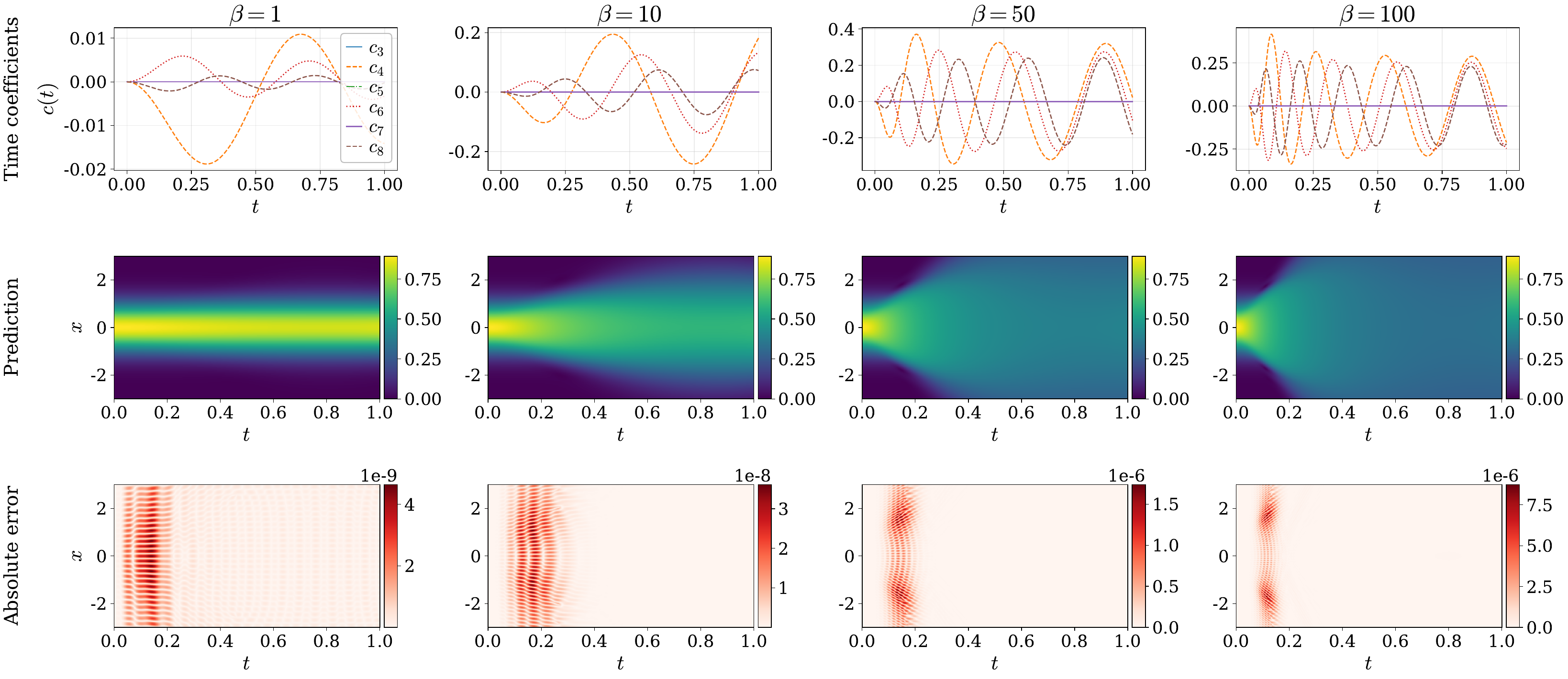}
\vspace{-2em}
\caption{Coefficient dynamics, SD-FSNN predictions, and pointwise absolute errors for $\beta\in\{1,10,50,100\}$.}
\label{fig:gpe_ct_multi_beta}
\end{figure}
\begin{figure}[t!]
\centering
\includegraphics[width=\textwidth]{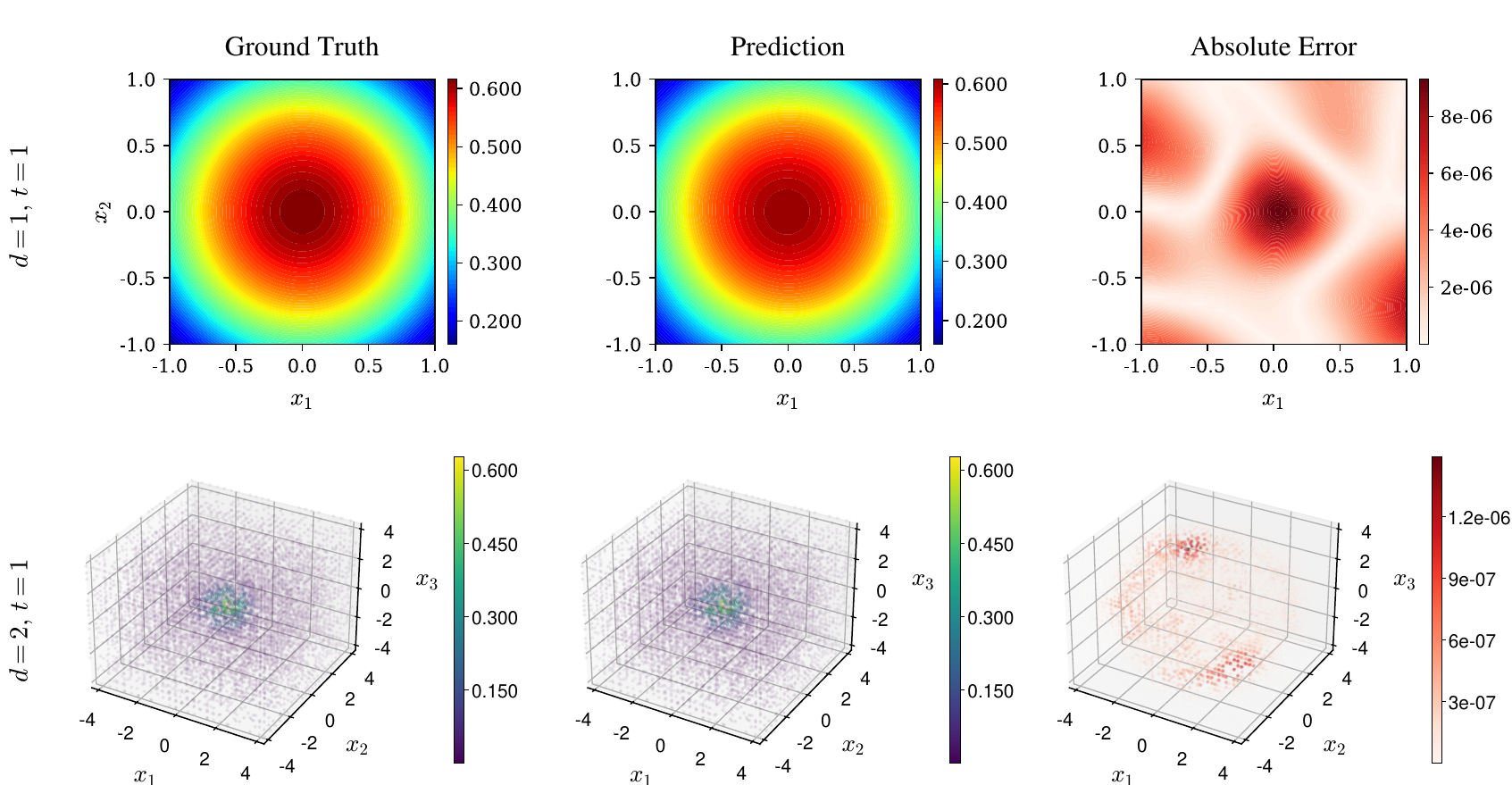}
\vspace{-2em}
\caption{Comparison of the reference solution, SD-FSNN prediction, and pointwise absolute error of $|\psi|$ for the 2D (1st row) and 3D (2nd row) GPEs with $\beta=10$.}
\label{fig:gpe_2d_beta10_T1}
\end{figure}

\Cref{fig:gpe_ct_multi_beta} visualizes the coefficient dynamics and SD-FSNN solutions for $\beta\in\{1,10,50,100\}$, where 
the first row plots the coefficients  $|c_3(t)|,\ldots,|c_8(t)|$ on $t\in[0,1]$.
%The displayed indices are restricted to $3$--$8$ so that the selected oscillatory modes can be distinguished.
%The second row shows the corresponding SD-FSNN predictions on the $(t,x)$ plane, and the third row shows the pointwise absolute errors. 
As can be seen, for $\beta=1$, the selected coefficients have small amplitudes and the predicted field remains localized near $x=0$.
As $\beta$ increases, the predicted field spreads more rapidly in space, and the selected coefficients exhibit larger amplitude and higher oscillations.
The error maps remain small for all tested $\beta$ values, indicating that SD-FSNN can work from weak to strong nonlinear regime.

We further demonstrate the high accuracy on 2D and 3D GPEs. 
%This test features nontrivial spatial spreading and interference patterns on an unbounded domain.
%The space-time separated representation and randomized spatial basis allow SD-FSNN to approximate the 2D wave function without a tensor-product grid.
\Cref{fig:gpe_2d_beta10_T1} shows the snapshot of $|\psi|^2$ at $t=1$.
The SD-FSNN prediction agrees with the reference solution in both cases, and the spatial error remains on the order of $10^{-6}$. 
The predictions match the references in support and amplitude, including the high-density core and decay toward the trap tails. 
These results again show that SD-FSNN can accurately capture the dynamics of GPE in the whole space.
%The largest pointwise errors occur near the center, where the density is largest.
%The peak error is of order $10^{-8}$ at $t=0$ and grows to the order of $10^{-6}$--$10^{-5}$ by $t=1$.
%No numerical instability are detected in all the tests. 

%This indicates stable short-time approximation of 2D condensate dynamics.

%\Cref{fig:gpe_3d_beta10_T1} shows subsampled points in $[-4,4]^3$ at $t \in \{0, 0.3, 0.6, 1\}$.

%\begin{figure}[htbp]
%\centering
%\includegraphics[width=\textwidth]{figs/gpe_3d_beta10_T1.pdf}
%\vspace{-1em}
%\caption{Comparison of the reference solution, SD-FSNN prediction, and pointwise absolute error of $|\psi|$ for the 3D GPE with $\beta=10$.}
%\label{fig:gpe_3d_beta10_T1}
%\end{figure}

Note that the classical tensor-product Hermite spectral discretizations in TSSP for references become very expensive in memory and computation. 
The proposed stochastic-dimension sampling in SD-FSNN avoids forming the full tensor-product basis while retaining an explicit whole-space representation.  
To assess dimensional scaling, we compare the runtime of SD-FSNN with TSSP  \cite{bao2005fourth} for $d=1,\ldots,8$.
\Cref{fig:run_time_tssp_sdfsnn} presents the computational time on a logarithmic scale.
For TSSP, we test spatial mesh size $h \in \{0.5, 0.25, 0.125, 0.0625\}$ with a fixed temporal step size $\Delta t = 10^{-3}$.
As the dimension increases, the TSSP cost grows exponentially: for all four tested mesh sizes, the wall-clock time at $d=8$ reaches $10^5$--$10^6$\,s.
Even with the coarsest mesh $h=0.5$, TSSP requires over $10^4$\,s already at $d=7$.
In contrast, SD-FSNN maintains a nearly flat cost curve, staying below $15$\,s for both $t=1$ and $t=10$.
At $d=8$ with $t=10$, SD-FSNN is four to five orders of magnitude faster than TSSP.
%This speedup has two sources.
%First, the stochastic-dimension structure avoids high-dimensional tensor-product grids.
%Second, the space-time separation reduces time integration to a coefficient ODE system.
These results indicate the alleviation of the dimensional scaling that limits traditional spectral methods for whole-space GPEs.

\begin{figure}[t!]
\centering
\includegraphics[width=.85\textwidth]{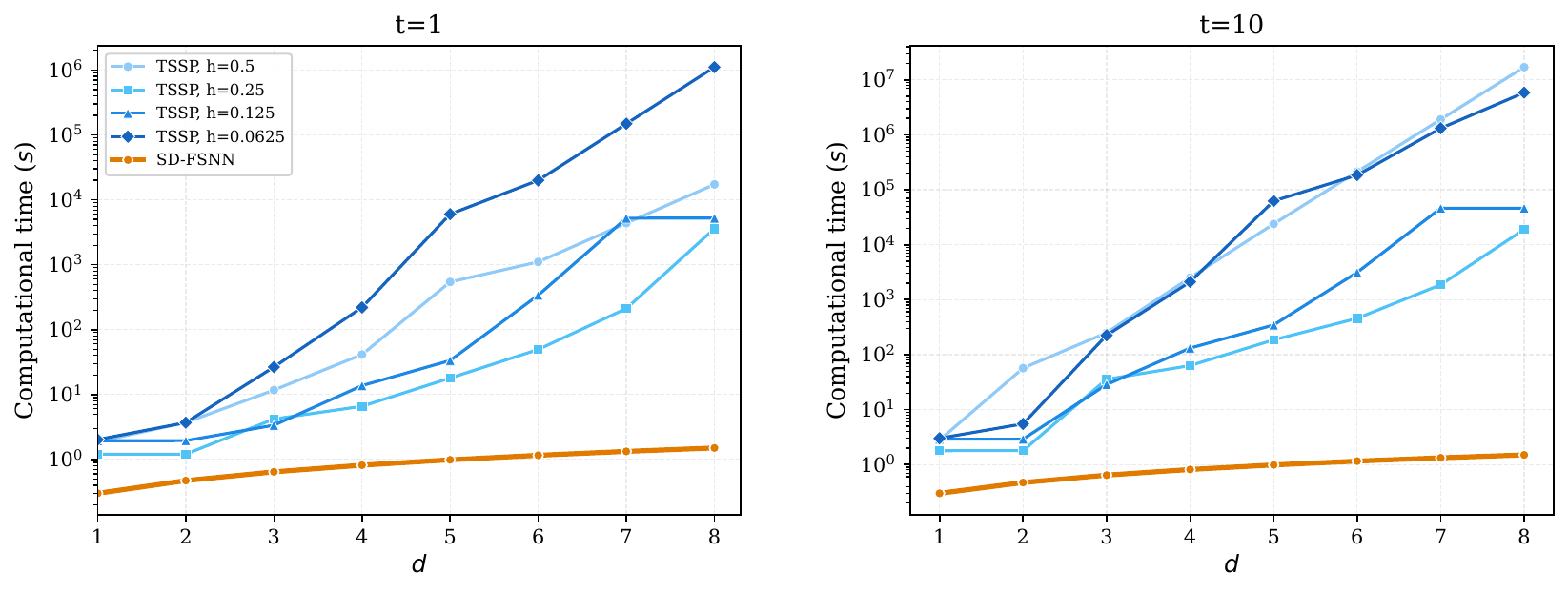}
\vspace{-1em}
\caption{Computational time versus spatial dimension $d$ for TSSP and SD-FSNN.}
\label{fig:run_time_tssp_sdfsnn}
\end{figure}

\subsection{Long-time performance}
\label{long_gpe}
We now investigate the performance of SD-FSNN over extended temporal domains by solving \eqref{eqn:GPE}   in 1D with $\beta=10$ and $\beta=30$ for $t\in[0,100]$ using SD-FSNN.  To address the importance of energy-projection in long-term case, we test SD-FSNN with both the full structure-preserving projection version and its mass-normalization-only variant.

% \cref{fig:gpe_sdre_T100_longtime} shows the SD-FSNN prediction, reference and pointwise absolute error on the $(t,x)$ plane.
% The pointwise error remains at the level of $10^{-7}$ and exhibits no rapid growth.
% \Cref{fig:gpe_sdre_T100_error} further quantifies the error in time. From the numerical results, we can observe the following. 
% After a brief initial transient from double-precision roundoff to roughly $10^{-7}$, the relative $L^2$ error plateaus and stays around order $10^{-6}$ through $T=100$ without exponential growth. 
% The effect of the energy projection depends on the interaction strength. 
% For the linear case $\beta=0$, the Hamiltonian is quadratic, and the mass-normalized trajectory already preserves this quadratic energy to high accuracy. 
% Therefore, adding the energy projection produces only negligible changes in the relative-error curve. 
% In contrast, for $\beta=10$, the Hamiltonian contains the nonlinear quartic interaction term, which is not controlled by mass normalization alone.
% Removing the energy projection then leads to a mild but visible increase in the long-time error.
\Cref{fig:gpe_sdre_T100_longtime} shows the SD-FSNN prediction, reference, and pointwise absolute error on the $(t,x)$ plane.
For $\beta=10$, the pointwise error stays at the level of $10^{-7}$, and for the case $\beta=30$ it remains at the level of $10^{-5}$. Indeed, stronger stiffness can induce a greater error, which is expected, while in both cases, no rapid growth is observed over $t\in[0,100]$.
\Cref{fig:gpe_sdre_T100_error} further quantifies the relative $L^2$ error in time.
After a brief initial transient driven by double-precision roundoff, the error of the full mass$+$energy projection variant quickly enters a plateau regime.
The plateau sits on the order $10^{-6}$ for $\beta=10$ and  $10^{-5}$ for $\beta=30$. 
%Through $T=100$, the error then exhibits at most an algebraic, sub-linear drift rather than exponential growth.
The SD-FSNN with mass-projection-only
is still working in long-time computing, while  
by comparison, 
the proposed full projection for SD-FSNN makes a clear difference. 
It can be seen that the energy projection brings about an improvement of one or two orders of magnitude. 
%For the strongly nonlinear case $\beta=30$, the quartic term is no longer controlled by mass normalization.
%Removing the energy projection then leads to a clearly larger long-time error.
%This error is roughly two orders of magnitude above the full variant.

These results indicate two points.
First, space-time separation with structure-preserving projections suppresses long-horizon error accumulation in this test. In contrast, 
standard global-in-time neural network solvers that lack causal time marching often do not share this property \cite{krishnapriyan2021characterizing,wang2021understanding}. 
Second, the additional energy projection indeed further benefits the long-term accuracy, especially for strong nonlinearity.

% These results suggest that space-time separation combined with structure-preserving projections suppresses long-horizon error accumulation in this test. 
% The comparison with the mass-only variant further indicates that the energy projection is most beneficial in the nonlinear regime. 
% In the linear case, its influence is almost invisible because the underlying quadratic invariant is already well preserved. 
% Global-in-time neural solvers that lack causal time marching do not share this property. 
\begin{figure}[t!]
\centering
\includegraphics[width=\textwidth]{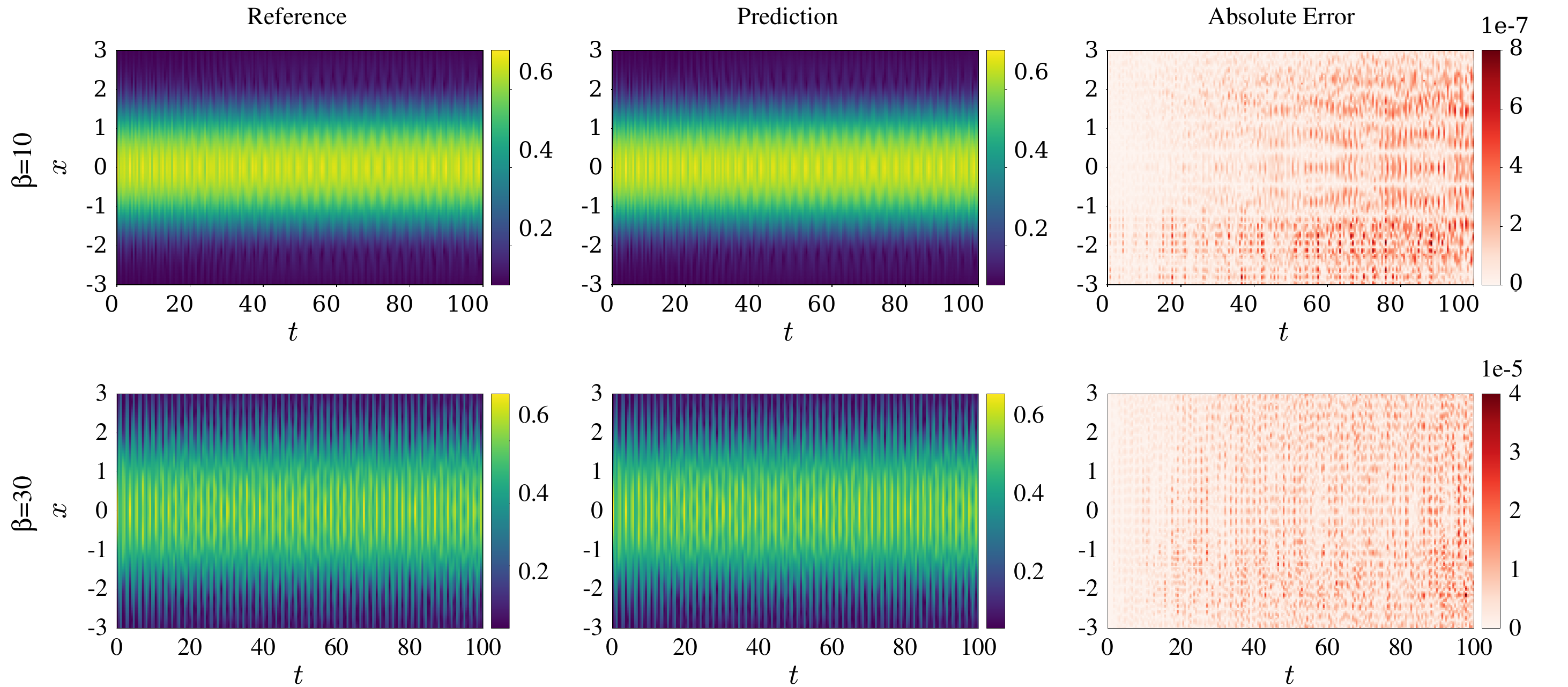}
\vspace{-2em}
\caption{SD-FSNN long-time prediction, reference solution, and pointwise absolute error for the 1D GPE on $t\in[0,100]$ with $\beta=10$ (1st row) and $\beta=30$ (2nd row).}
\label{fig:gpe_sdre_T100_longtime}
\end{figure}
\begin{figure}[t!]
\centering
\includegraphics[width=.9\textwidth]{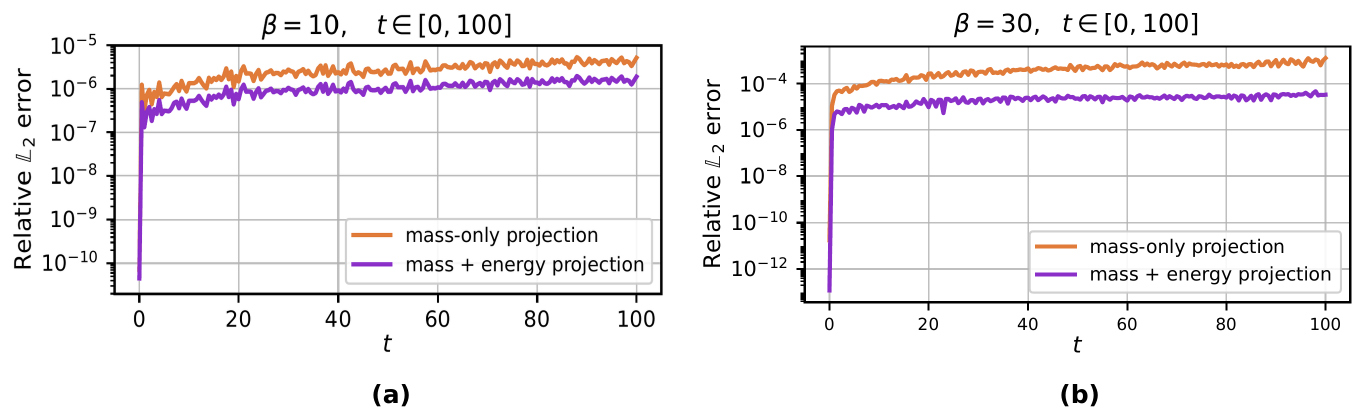}
\vspace{-1em}
\caption{Relative $L^2$ error of the SD-FSNN long-time prediction for the 1D GPE on $t\in[0,100]$ with $\beta=10$ and $\beta=30$. The comparison also illustrates the effect of removing the energy projection from the structure-preserving correction.}
\label{fig:gpe_sdre_T100_error}
\end{figure}

\subsection{Ablation study}\label{sec:ablation}
We perform ablation studies on the key ingredients proposed for our SD-FSNN method: decay envelope, stochastic-dimension sampling, and structure-preserving projections introduced in \Cref{sec:sdfsnn}. Tests here are done on the same 1D GPE with $\beta=10$. 

\paragraph{\textbf{Effect of the exponential decay strategy}}
We first assess the decay envelope introduced in \cref{ansatz}.
SD-FSNN uses $\rho(\bm{x})$ to impose $\lvert \psi(\bm{x}, t)\rvert \to 0$ as $\lVert \bm{x}\rVert \to \infty$.
\Cref{fig:ablation_decay}(a) reports the relative $L^2$ error as a function of the number of basis functions for several decay rates $\alpha$.
Without the envelope, the error stagnates near $\mathcal{O}(1)$.
The decay-matched value $\alpha=0.5$ decreases the error to approximately $10^{-8}$ using $512$ basis. The nearby values can also work, but less effectively.  
\Cref{fig:ablation_decay}(b) shows the de-enveloped initial data $\rho^{-1}\psi_0$ under decay-matched $\alpha=0.5$.
Unlike tuned nearby values such as $\alpha=0.3$, this choice makes $\rho^{-1}\psi_0$ constant, so the network does not need to learn any residual spatial decay.
\Cref{fig:ablation_decay}(c) shows the temporal error.
The ansatz without decay becomes unstable, whereas the decay-envelope model remains accurate over the tested interval.

% \paragraph{\textbf{Effect of the exponential decay strategy}}
% We first assess the decay envelope introduced in \cref{ansatz}.
% SD-FSNN uses $\rho(\bm{x})$ to impose $\lvert \psi(\bm{x}, t)\rvert \to 0$ as $\lVert \bm{x}\rVert \to \infty$.
% \Cref{fig:ablation_decay}(a) reports the relative $L^2$ error as a function of the number of basis functions for several decay parameters $\alpha$.
% Without the envelope, the error stagnates near $\mathcal{O}(1)$.
% With $\alpha=0.5$, the error decreases to approximately $10^{-8}$ using $512$ basis functions.
% \Cref{fig:ablation_decay}(b) shows the temporal error.
% The ansatz without decay becomes unstable, whereas the decay-envelope model remains accurate over the tested interval.

\begin{figure}[t!]
\centering
\includegraphics[width=\textwidth]{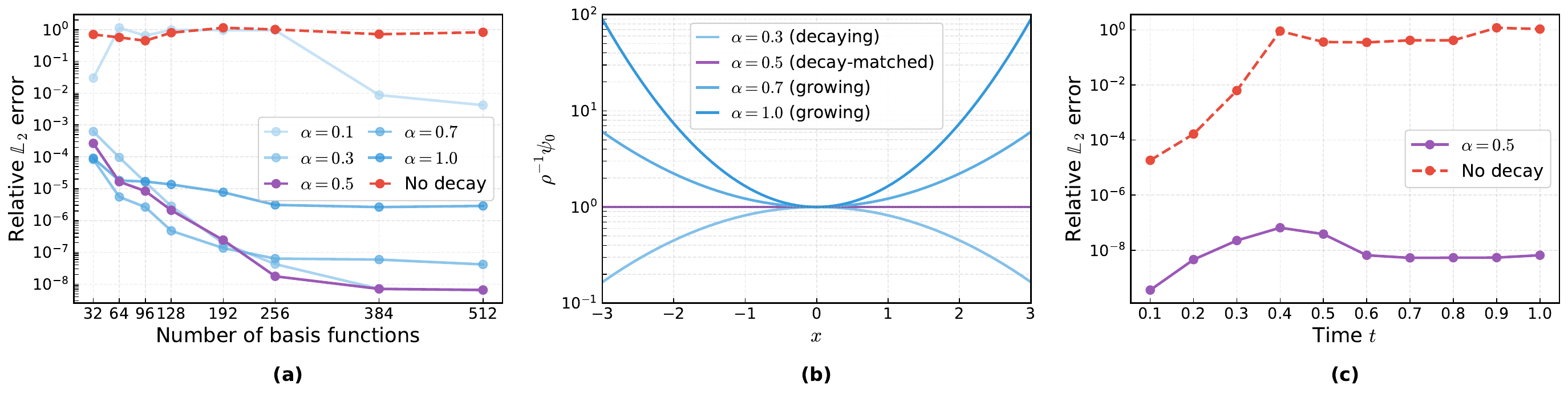}
\vspace{-2em}
\caption{Effect of the exponential decay strategy. (a) Error convergence for different $\alpha$. (b) De-enveloped initial data $\rho^{-1}\psi_0$, with $\alpha=0.5$ decay-matched. (c) Temporal error with and without the decay envelope.}
\label{fig:ablation_decay}
\end{figure}
\begin{figure}[t!]
\centering
\includegraphics[width=.9\textwidth]{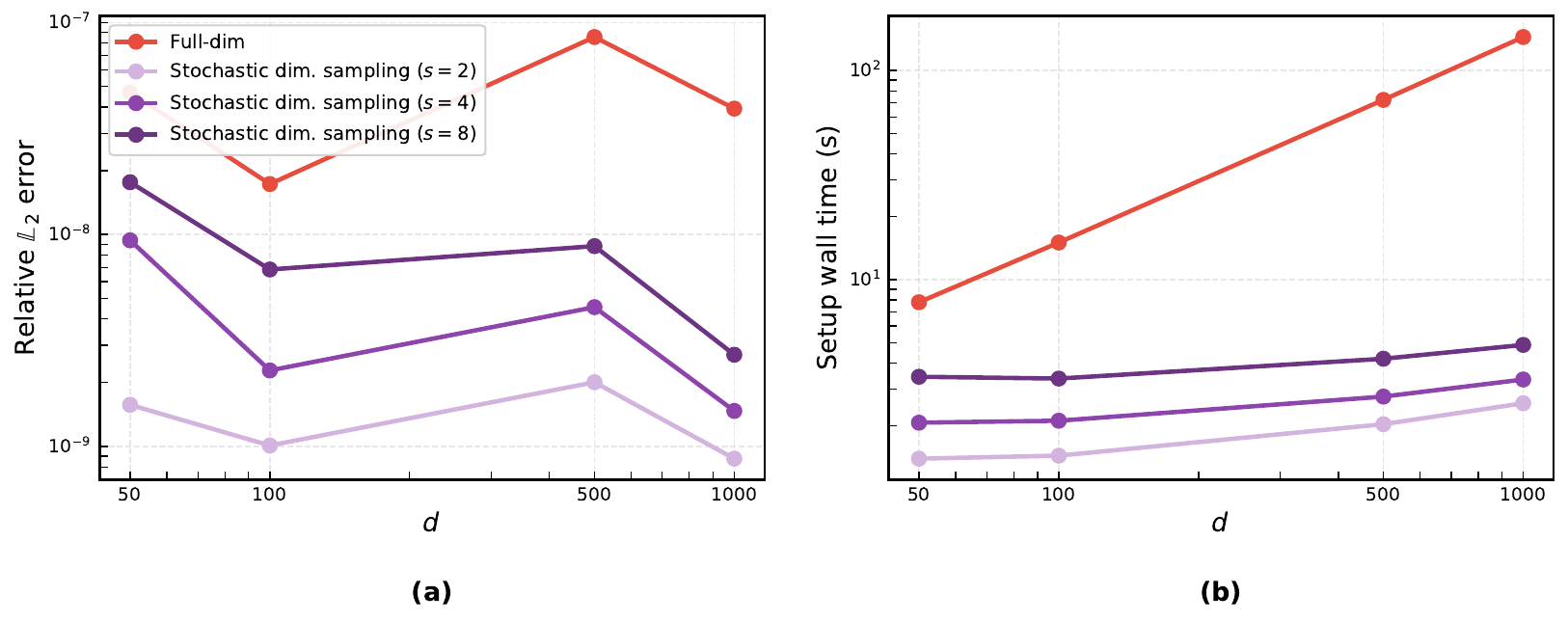}
\vspace{-1em}
\caption{Effect of stochastic-dimension sampling. (a) Relative $L^2$ error and (b) setup wall-clock time for full-dimension evaluation and stochastic sampling with different subset sizes $s$.}
\label{fig:ablation_sd}
\end{figure}
\begin{figure}[t!]
\centering
\includegraphics[width=0.9\textwidth]{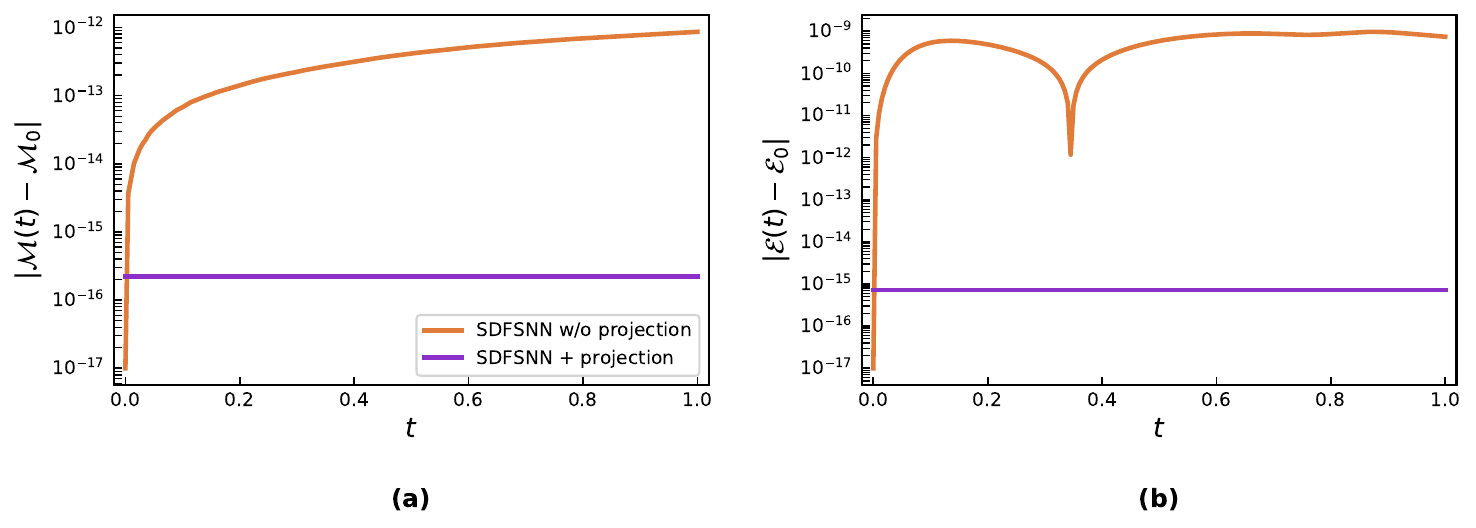}
\vspace{-1em}
\caption{Effect of the structure-preserving projections. (a) Deviation of the discrete mass, $|\mathcal{M}(t)-\mathcal{M}_0|$. (b) Deviation of the discrete energy, $|\mathcal{E}(t)-\mathcal{E}_0|$.}
\label{fig:ablation_conservation}
\end{figure}

\paragraph{\textbf{Effect of stochastic-dimension sampling}}
We then assess the accuracy--cost trade-off from stochastic-dimension sampling by comparing it with full-dimension evaluation for different subset sizes $m$.
\Cref{fig:ablation_sd}(a) shows the relative $L^2$ error for $d \in \{50,100,500,1000\}$.
Stochastic sampling preserves accuracy in these tests and can yield smaller errors than the full-dimension baseline, in line with the conditional unbiasedness in \Cref{thm:sdfsnn_unbiased_gpe} and the variance estimate in \Cref{lem:sdfsnn_variance}.
\Cref{fig:ablation_sd}(b) reports the setup wall-clock time.
The full-dimension cost increases with $d$, whereas stochastic sampling keeps setup time nearly flat and smaller, matching the $\mathcal{O}(Ms)$ per-evaluation cost predicted by \cref{eq:sdfsnn_laplacian}.

\paragraph{\textbf{Effect of structure-preserving projections}}
We next examine the closed-form two-step mass and energy projection in \Cref{fig:ablation_conservation}.
\Cref{fig:ablation_conservation}(a) reports the deviation of the discrete mass $\mathcal{M}(t)$ \eqref{eq:mass_disc} from its initial value $\mathcal{M}_0$.
\Cref{fig:ablation_conservation}(b) reports the corresponding deviation of the discrete energy $\mathcal{E}(t)$ \eqref{eq:energy_disc} (replace $\mathcal{E}(C(t))$ for short).
After the mass projection step \eqref{eq:step1_mass}, the mass error $|\mathcal{M}(t)-\mathcal{M}_0|$ stays at the level of double-precision roundoff ($10^{-16}$).
%Here $\mathcal{M}_0=1$, and the observed error is about $10^{-16}$.
The energy error $|\mathcal{E}(t)-\mathcal{E}_0|$ with $\mathcal{E}_0$ the initial value, is maintained similarly near the roundoff, about $10^{-15}$. 
These results confirm that the projection enforces the prescribed discrete invariants and suppresses non-physical drift over time.

\section{More exploration and comparison}
\label{sec:discussion}
This section explores the performance of SD-FSNN beyond low-dimensional tests. We first apply the method to extremely high-dimensional benchmarks and compare its performance with existing representative solvers. Then, we use the Kolmogorov $n$-width perspective \cite{peherstorfer2022breaking} to clarify when the frozen-basis representation remains effective and where its limitations may arise.

\subsection{Extremely high dimension}
We manufacture an exact solution $\psi_{\text{ex}}(\bm{x})$ for the high-dimensional nonlinear equation  
by adding a source term $g(\bm{x})$ to the original GPE (see also \cite{hu2024tackling}): %MMS yields a modified problem whose solution is known by construction.
%All methods solve the same problem on the unit ball $\mathbb{B}^d$ with Dirichlet data $\psi_{\text{exact}}$ and source $g(\bm{x})$. 
%SD-FSNN uses the isotropic Gaussian envelope matched to \cref{eqn:gpe_exact}, so the de-enveloped variable equals $\mathcal{P}(\bm{x})$.
%The PINN-type and tensor-network baselines retain their native parameterizations.
%We prescribe the exact solution as
\begin{equation}\label{eqn:gpe_exact}
\begin{split}
    &\psi_{\text{ex}}(\bm{x}) =  \left(\prod_{j=1}^{d}\frac{\gamma_j}{\pi}\right)^{1/4}\mathrm{e}^{-\frac{1}{2}\sum_{j=1}^{d}\gamma_j x_j^2}\mathcal{P}(\bm{x}),\\
    &\mathcal{P}(\bm{x})=\left[ 1 + \epsilon \sum_{i=1}^{d-1} \xi_i \sin\left(x_i + \cos(x_{i+1}) + x_{i+1}\cos(x_i)\right)\right],
    \end{split}
\end{equation}
that solves  
the manufactured (stationary) problem  
% Given \cref{eqn:gpe_exact}, the corresponding manufactured stationary problem is posed on the unit ball $\mathbb{B}^d=\{\bm{x}\in\mathbb{R}^d:\|\bm{x}\|<1\}$. The boundary data on $\partial\mathbb{B}^d$ are prescribed by $\psi_{\text{exact}}$:
\begin{equation}\label{eqn:gpe_mms}
% \left\{
% \begin{aligned}
\frac{1}{2}\nabla^2\psi = V_d(\bm{x})\psi + \beta|\psi|^2\psi + g(\bm{x}), \qquad \bm{x}\in \mathbb{R}^d,
% \end{aligned}
% \right.
\end{equation}
where $g(\bm{x}) = \frac{1}{2}\nabla^2\psi_{\text{ex}} - V_d(\bm{x})\psi_{\text{ex}}-\beta|\psi_{\text{ex}}|^2\psi_{\text{ex}}$. We take $\beta=10$, $\gamma_j=2$ for $1\leq j\leq d$,  $\epsilon=10^{-3}$ and $\xi_i$ for $1\leq i\leq d-1$ generated from $\mathcal{N}(0, 1)$.
% Here $g(\bm{x}) = \frac{1}{2}\nabla^2\psi_{\text{exact}} - V_d(\bm{x})\psi_{\text{exact}}-\beta|\psi_{\text{exact}}|^2\psi_{\text{exact}}$.
Using this manufactured problem, we compare SD-FSNN with eight baselines from the high-dimensional PDE literature. These are SDGD \cite{hu2024tackling}, RS-PINNs \cite{hu2025smooth}, HTE \cite{hu2024hte}, STDE \cite{shi2024stde}, ForLap \cite{li2024forlap}, Score-PINNs \cite{hu2025score}, TNN \cite{wang2022tensor}, and TNN-RBF \cite{wang2025tensor}. 
For baseline methods on a bounded domain, the GPE is solved on $\mathbb{B}^d = \{\bm{x}\in\mathbb{R}^d: \|\bm{x}\|<1\}$ with Dirichlet boundary data
$
\psi(\bm{x}) = \psi_{\text{ex}}(\bm{x}), \quad \bm{x}\in \partial \mathbb{B}^d.$ 
%while the exact solution itself is defined on $\mathbb{R}^d$ with exponential decay at infinity.
All methods except TNN and TNN-RBF use stochastic-dimension sampling with $s=4$. 
The numerical results are given in \Cref{tab:gpe_static_results}, where $N_f$ is the number of residual samples per iteration and $N_c$ is the number of collocation points used for the least-squares system (TNN-RBF and SD-FSNN).

The quantitative comparison in \Cref{tab:gpe_static_results} shows that SD-FSNN achieves the lowest errors and the shortest wall-clock time across the reported tests.
In particular, it achieves the lowest $L^2$ and $L^1$ errors for $d \in \{10,100,1000\}$, with an $L^2$ error of $3.45\times10^{-4}$ at $d=1000$.
By contrast, the randomized-estimator, score-based, and tensor-network baselines remain less accurate. Several show increasing errors as the ambient dimension grows.
This indicates that the frozen sampled basis and space-time separated coefficient evolution remain effective in the high-dimensional regime; the conditional unbiasedness in \Cref{thm:sdfsnn_unbiased_gpe} ensures that the dimension-batch sampler introduces no systematic bias as $d$ grows, while the variance prefactor in \cref{eq:sdfsnn_variance} (\Cref{lem:sdfsnn_variance}) controls the residual sampling fluctuation.
Concerning the computational cost, 
 SD-FSNN takes $12.18$\,s at $d=1000$, while by comparison, ForLap requires $259.17$\,s, TNN-RBF $289.05$\,s, RS-PINNs $696.48$\,s, and SDGD, STDE, HTE, and Score-PINNs each exceed $1500$\,s.
Moreover, TNN and TNN-RBF exhibit wall-clock-time growth that is linear in $d$.
Thus, SD-FSNN is at least $21\times$ faster than the fastest baseline at $d=1000$. %It also attains the lowest errors among the methods in \Cref{tab:gpe_static_results}.

\begin{table}[t!]
\centering
\caption{Comparison of methods on the static high-dimensional GPE \eqref{eqn:gpe_mms}.}
\label{tab:gpe_static_results}
\scriptsize
\setlength{\tabcolsep}{8pt}
\begin{tabular}{ccccccc}
\toprule
Dim & Method & $N_f$ & $N_c$ & $L^2$ Error & $L^1$ Error & Time (s) \\
\midrule
10 & SDGD & 1000 & -- & $4.42\times10^{-3}$ & $3.14\times10^{-3}$ & 199.53 \\
10 & STDE & 1000 & -- & $4.42\times10^{-3}$ & $3.14\times10^{-3}$ & 198.11 \\
10 & HTE & 1000 & -- & $6.52\times10^{-3}$ & $4.54\times10^{-3}$ & 617.58 \\
10 & RS-PINNs & 1000 & -- & $5.48\times10^{-3}$ & $3.86\times10^{-3}$ & 140.32 \\
10 & ForLap & 1000 & -- & $2.33\times10^{-3}$ & $1.70\times10^{-3}$ & 154.25 \\
10 & Score-PINNs & 1000 & -- & $1.45\times10^{-2}$ & $1.16\times10^{-2}$ & 247.09 \\
10 & TNN & 1000 & -- & $6.04\times10^{-2}$ & $4.60\times10^{-2}$ & 25.67 \\
10 & TNN-RBF & 1000 & 128 & $3.49\times10^{-2}$ & $2.33\times10^{-2}$ & 22.44 \\
10 & \textbf{SD-FSNN} & 1000 & 128 & {\boldmath $2.41\times10^{-6}$} & {\boldmath $1.58\times10^{-6}$} & \textbf{2.12} \\
\midrule
100 & SDGD & 1000 & -- & $7.35\times10^{-3}$ & $5.31\times10^{-3}$ & 836.11 \\
100 & STDE & 1000 & -- & $7.35\times10^{-3}$ & $5.31\times10^{-3}$ & 848.18 \\
100 & HTE & 1000 & -- & $1.03\times10^{-2}$ & $8.25\times10^{-3}$ & 647.88 \\
100 & RS-PINNs & 1000 & -- & $8.44\times10^{-3}$ & $5.50\times10^{-3}$ & 381.23 \\
100 & ForLap & 1000 & -- & $2.32\times10^{-3}$ & $2.10\times10^{-3}$ & 1016.21 \\
100 & Score-PINNs & 1000 & -- & $5.17\times10^{-2}$ & $5.19\times10^{-2}$ & 342.74 \\
100 & TNN & 1000 & -- & $1.36\times10^{-2}$ & $1.01\times10^{-2}$ & 122.84 \\
100 & TNN-RBF & 1000 & 128 & $5.16\times10^{-3}$ & $3.55\times10^{-3}$ & 54.76 \\
100 & \textbf{SD-FSNN} & 1000 & 128 & {\boldmath $8.05\times10^{-5}$} & {\boldmath $6.98\times10^{-5}$} & \textbf{6.21} \\
\midrule
1000 & SDGD & 1000 & -- & $7.43\times10^{-3}$ & $5.87\times10^{-3}$ & 1573.35 \\
1000 & STDE & 1000 & -- & $7.43\times10^{-3}$ & $5.87\times10^{-3}$ & 2029.26 \\
1000 & HTE & 1000 & -- & $1.29\times10^{-2}$ & $6.33\times10^{-3}$ & 1638.89 \\
1000 & RS-PINNs & 1000 & -- & $1.93\times10^{-2}$ & $1.59\times10^{-2}$ &  1115.30\\
1000 & ForLap & 1000 & -- & $5.26\times10^{-3}$ & $4.82\times10^{-3}$ & 181.41 \\
1000 & Score-PINNs & 1000 & -- & $4.61\times10^{-1}$ & $4.41\times10^{-1}$ & 1571.62 \\
1000 & TNN & 1000 & -- & $1.63\times10^{-3}$ & $1.12\times10^{-3}$ & 979.35 \\
1000 & TNN-RBF & 1000 & 128 & $1.57\times10^{-3}$ & $1.19\times10^{-3}$ & 289.05 \\
1000 & \textbf{SD-FSNN} & 1000 & 128 & {\boldmath $3.45\times10^{-4}$} & {\boldmath $2.58\times10^{-4}$} & \textbf{12.18} \\
\bottomrule
\end{tabular}
\end{table}

\subsection{Kolmogorov n-width barrier}

The preceding observation that SD-FSNN remains efficient as the dimension grows needs to be viewed in light of the Kolmogorov $n$-width \cite{peherstorfer2022breaking}. Indeed, the spatial basis in the basic SD-FSNN is frozen, and only the output coefficients evolve.
Its approximation capacity is therefore limited by the Kolmogorov $n$-width of the target manifold, a known bottleneck for static randomized representations \cite{berman2024randomized,kast2024positional}.
For solution manifolds with high intrinsic complexity, frozen spectral or randomized reduced models may still require many degrees of freedom. 
This limitation also applies to the static-basis baselines in \Cref{fig:kolmogorov_nwidth}, including RFM, LocalELM, TNN, and GNN \cite{ainsworth2021gnn,ainsworth2025egnn} that fits radial-basis Gaussian centers.
The present results do not remove this barrier.

To probe pairwise structure separately from dimensional cost, we replace the previous $\tanh$ basis by a structured sparse Fourier basis,
\begin{equation*}%\label{eq:sparse_fourier}
\phi_m(\bm{x})=\cos(\bm{\omega}_m^\top \bm{x}_{S_m}+\eta_m),\quad |S_m|=2,
\end{equation*}
with $S_m\subset\{1,\dots,d\}$ a random coordinate pair and $(\bm{\omega}_m,\eta_m)$ random frequency/phase parameters. This basis encodes low-order pairwise interactions matched to $\mathcal{P}(\bm{x})$ in \cref{eqn:gpe_exact}. The rest of the SD-FSNN pipeline is unchanged, and we compute again the problems in \Cref{tab:gpe_static_results}.

For $d=10$ with the sparse Fourier basis, SD-FSNN reaches $1.60\times 10^{-10}$ at $M=4096$, while the tested baselines yield substantially larger errors.
With $M=150d$, the error remains at roughly $\mathcal{O}(10^{-8})$ for $d=5$ to $50$. The basis size grows linearly in $d$, so the error (not the cost) is dimension-insensitive.
Thus, in this structured setting, SD-FSNN alleviates the practical impact of the Kolmogorov $n$-width barrier.
It also avoids some of the nonconvex optimization difficulties associated with fully trainable PINNs \cite{krishnapriyan2021characterizing,wang2021understanding,wang2022and}.
For manifolds without such sparsity, periodically resampling the frozen basis (or extend via higher-order interaction terms) remains a natural extension.

\begin{figure}[t!]
\centering
\includegraphics[width=0.9\textwidth]{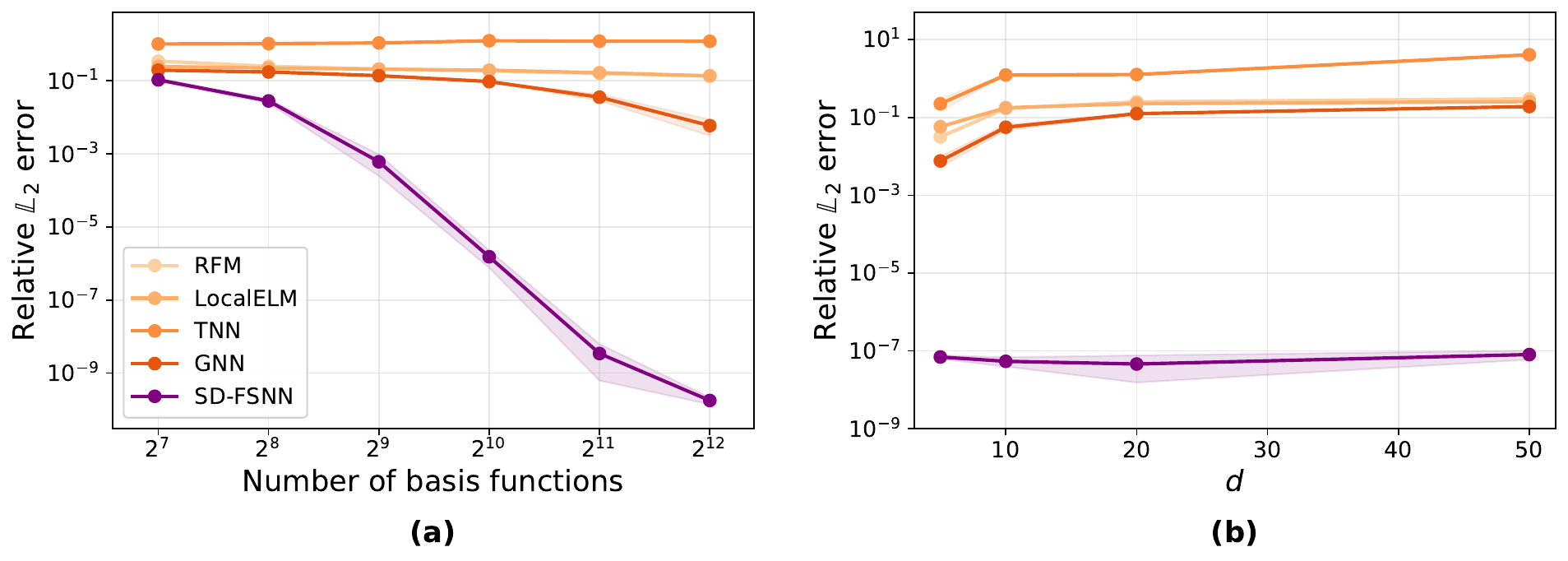}
\vspace{-1em}
\caption{Analysis of the Kolmogorov $n$-width barrier. (a) Relative $L^2$ error versus the number of basis functions $M$ at $d=10$. (b) Relative $L^2$ error versus dimension $d$ with $M$ scaled as $150d$.}
\label{fig:kolmogorov_nwidth}
\end{figure}

\section{Conclusion}
\label{sec:conclusion}
We introduced SD-FSNN, a randomized neural network framework for solving high-dimensional GPE on unbounded domain. The method tackles the curse-of-dimensionality through three key components: (i) a gradient-free, space-time separated architecture using a prescribed exponential-decay envelope and frozen features; (ii) a conditionally unbiased stochastic-dimension sampler that essentially reduces operator evaluation cost; and (iii) a closed-form projection preserving both mass and energy, ensuring long-term stability with theoretically proven uniform boundedness.
Numerical experiments on GPEs up to 1000 dimensions demonstrate that SD-FSNN achieves orders-of-magnitude higher accuracy and significantly lower computational cost compared to state-of-the-art methods. 
The method is generalizable to other high-dimensional evolution problems on unbounded domains, which will be addressed in future work.
%The numerical results suggest that high-dimensional GPEs on unbounded domains provide a useful benchmark for conditionally unbiased randomized high-dimensional solvers.
%The proposed space-time separation strategy facilitates accurate long-time dynamics in the tested regime.

%Potential limitations include the trade-off between Monte Carlo subset size, computational cost, and variance.
%The representational capability of the frozen basis may also diminish as the nonlinear interaction coefficient increases.
%Future work may therefore explore variance reduction techniques and compensation strategies for large interaction coefficients.
%Overall, SD-FSNN represents a step toward practical neural PDE solvers for high-dimensional problems on unbounded domains.

\section*{Acknowledgments}
The author thanks Xiaofei Zhao for carefully reading and revising the manuscript and for providing helpful suggestions. 
The author also thanks Tingfeng Wang for developing the code used in the numerical solution.

\bibliographystyle{siamplain}
\bibliography{references}

% \newpage

% \appendix

\end{document}

%% file: ex_shared.tex
\usepackage{lipsum}
\usepackage{amsfonts}
\usepackage{graphicx}
\usepackage{epstopdf}
\usepackage{algorithmic}
\usepackage{booktabs}
\usepackage{multirow} 
\usepackage{hyperref}
\usepackage{physics}
\usepackage{bbm}
\usepackage{bm}
\usepackage{tikz, tikz-cd}
\usetikzlibrary{positioning, arrows.meta} 

\usepackage[normalem]{ulem}
\ifpdf
  \DeclareGraphicsExtensions{.eps,.pdf,.png,.jpg}
\else
  \DeclareGraphicsExtensions{.eps}
\fi

\newsiamremark{remark}{Remark}
\newsiamremark{hypothesis}{Hypothesis}
\crefname{hypothesis}{Hypothesis}{Hypotheses}
\newsiamthm{claim}{Claim}

\headers{SD-FSNN for high-dim whole-space GPE}{Z. Liang}

\title{Stochastic-Dimension Frozen Sampled Neural Network for high-dimensional Gross-Pitaevskii equations on unbounded domains
\thanks{Submitted to the editors \today.
\funding{This work was supported by the National Natural Science Foundation of China (NO. 12202157).}}}

\author{Zhangyong Liang\thanks{National Center for Applied Mathematics, Tianjin University, Tianjin, 300072, China. (\email{zyliang1994@tju.edu.cn})}.}
\usepackage{amsopn}

\makeatletter
\newcommand*{\addFileDependency}[1]{% argument=file name and extension
  \typeout{(#1)}% latexmk will find this if $recorder=0 (however, in that case, it will ignore #1 if it is a .aux or .pdf file etc and it exists! if it doesn't exist, it will appear in the list of dependents regardless)
  \@addtofilelist{#1}% if you want it to appear in \listfiles, not really necessary and latexmk doesn't use this
  \IfFileExists{#1}{}{\typeout{No file #1.}}% latexmk will find this message if #1 doesn't exist (yet)
}
\makeatother

%% file: references.bib
@article{erdHos2010derivation,
  title={Derivation of the {G}ross-{P}itaevskii equation for the dynamics of {B}ose-{E}instein condensate},
  author={Erd{\H{o}}s, L{\'a}szl{\'o} and Schlein, Benjamin and Yau, Horng-Tzer},
  journal={Ann. Math.},
  pages={291--370},
  year={2010},
  publisher={JSTOR}
}

@article{GPEreview,
  title={Computational methods for the dynamics of the nonlinear {S}chr\"odinger/{G}ross–{P}itaevskii equations},
  author={Antoine, X and Bao, W and Besse, C},
  journal={Comput. Phys. Comm.},
  volume={184},
  pages={2621--2633},
  year={2013}
}

@article{DG,
  title={Local discontinuous galerkin methods for nonlinear Schr\"odinger equations},
  author={Xu, Y and Shu, C.W.},
  journal={J. Comput. Phys.},
  volume={205},
  pages={72--97},
  year={2005}
}

@book{lieb2005math,
  title={The mathematics of the {B}ose gas and its condensation},
  author={Lieb, Elliott H and Solovej, Jan Philip and Seiringer, Robert and Yngvason, Jakob},
  year={2005},
  publisher={Springer}
}

@inproceedings{gallicchio2020rann,
  title={Deep randomized neural networks},
  author={Gallicchio, Claudio and Scardapane, Simone},
  booktitle={Recent Trends in Learning From Data: Tutorials from the INNS Big Data and Deep Learning Conference (INNSBDDL2019)},
  pages={43--68},
  year={2020},
  organization={Springer}
}

@article{huybrechs2024rann,
  title={Sigmoid functions, multiscale resolution of singularities, and hp-mesh refinement},
  author={Huybrechs, Daan and Trefethen, Lloyd N},
  journal={SIAM Rev.},
  volume={66},
  number={4},
  pages={683--693},
  year={2024},
  publisher={SIAM}
}

@book{Cazenave,
  title={Semilinear {S}chr\"odinger Equations},
  author={Cazenave, T.},
  year={2003},
  publisher={Courant Lect. Notes Math., 10, Amer. Math. Soc.}
}

@article{georgescu2014quantum,
  author  = {I. M. Georgescu and S. Ashhab and Franco Nori},
  title   = {Quantum simulation},
  journal = {Rev. Modern Phys.},
  volume  = {86},
  number  = {1},
  pages   = {153--185},
  year    = {2014},
  doi     = {10.1103/RevModPhys.86.153}
}

@article{huang2012extreme,
  author  = {Guang-Bin Huang and Hongming Zhou and Xiaojian Ding and Rui Zhang},
  title   = {Extreme Learning Machine for Regression and Multiclass Classification},
  journal = {IEEE Transactions on Systems, Man, and Cybernetics, Part B (Cybernetics)},
  volume  = {42},
  number  = {2},
  pages   = {513--529},
  year    = {2012},
  doi     = {10.1109/TSMCB.2011.2168604}
}

@article{anglin2002bose,
  title={Bose--{E}instein condensation of atomic gases},
  author={Anglin, James R and Ketterle, Wolfgang},
  journal={Nature},
  volume={416},
  number={6877},
  pages={211--218},
  year={2002},
  publisher={Nature Publishing Group UK London}
}

@book{spectralbook,
  title={Spectral Methods: Algorithms, Analysis and Applications},
  author={Shen, J. Tang, T. and Wang, L.},
  year={2011},
  publisher={Springer, Berlin}
}

@article{fem,
  title={High-order mass-and energy-conserving {SAV}-{G}auss collocation finite element methods for the nonlinear {S}chr\"odinger equation},
  author={Feng, X and Li, B and Ma, S},
  journal={SIAM J. Numer. Anal.},
  volume={59},
  pages={1566--1591},
  year={1961}
}

@article{ShenFOCM,
  title={Error analysis of the {S}trang time-splitting {L}aguerre–{H}ermite/{H}ermite collocation methods for the {G}ross–{P}itaevskii equation},
  author={Shen, J and Wang, Z},
  journal={Found. Comput. Math.},
  volume={13},
  pages={99--137},
  year={2013}
}

@article{bao2005fourth,
  title={A fourth-order time-splitting {L}aguerre--{H}ermite pseudospectral method for {B}ose--{E}instein condensates},
  author={Bao, Weizhu and Shen, Jie},
  journal={SIAM J. Sci. Comput.},
  volume={26},
  number={6},
  pages={2010--2028},
  year={2005},
  publisher={SIAM}
}

@article{chou2023adaptive,
  title={Adaptive {H}ermite spectral methods in unbounded domains},
  author={Chou, Tom and Shao, Sihong and Xia, Mingtao},
  journal={Appl. Numer. Math.},
  volume={183},
  pages={201--220},
  year={2023},
  publisher={Elsevier}
}

@article{xu2006absorb,
  title={Absorbing boundary conditions for nonlinear {S}chr{\"o}dinger equations},
  author={Xu, Zhenli and Han, Houde},
  journal={Phys. Rev. E},
  volume={74},
  number={3},
  pages={037704},
  year={2006},
  publisher={American Physical Society}
}

@article{zhang2008unif,
  title={Unified approach to split absorbing boundary conditions for nonlinear {S}chr{\"o}dinger equations},
  author={Zhang, Jiwei and Xu, Zhenli and Wu, Xiaonan},
  journal={Phys. Rev. E},
  volume={78},
  number={2},
  pages={026709},
  year={2008},
  publisher={APS}
}

@article{PMLreview,
  title={A review of transparent and artificial boundary conditions techniques for linear and nonlinear {S}chr\"odinger equations},
  author={Antoine, Xavier and Arnold, A and Besse, C and Ehrhardt, M and Sch\"adle, A},
  journal={Commun. in Comput. Physics},
  volume={4},
  pages={729-796},
  year={2008}
}

@article{hu2024hte,
  title={Hutchinson trace estimation for high-dimensional and high-order physics-informed neural networks},
  author={Hu, Zheyuan and Shi, Zekun and Karniadakis, George Em and Kawaguchi, Kenji},
  journal={Comput. Methods Appl. Mech. Engrg.},
  volume={424},
  pages={116883},
  year={2024},
  publisher={Elsevier}
}

@article{wang2022tensor,
  title={Tensor neural network and its numerical integration},
  author={Wang, Yifan and Jin, Pengzhan and Xie, Hehu},
  journal={arXiv preprint arXiv:2207.02754},
  year={2022}
}

@article{cho2022spinn,
  title={Separable {PINN}: Mitigating the curse of dimensionality in physics-informed neural networks},
  author={Cho, Junwoo and Nam, Seungtae and Yang, Hyunmo and Yun, Seok-Bae and Hong, Youngjoon and Park, Eunbyung},
  journal={arXiv preprint arXiv:2211.08761},
  year={2022}
}

@article{han2018hjb,
  title={Solving high-dimensional partial differential equations using deep learning},
  author={Han, Jiequn and Jentzen, Arnulf and E, Weinan},
  journal={Proc. Natl. Acad. Sci. },
  volume={115},
  number={34},
  pages={8505--8510},
  year={2018},
  publisher={National Academy of Sciences}
}

@article{beck2019machine,
  title={Machine learning approximation algorithms for high-dimensional fully nonlinear partial differential equations and second-order backward stochastic differential equations},
  author={Beck, Christian and E, Weinan and Jentzen, Arnulf},
  journal={J. Nonlinear Sci.},
  volume={29},
  number={4},
  pages={1563--1619},
  year={2019},
  publisher={Springer}
}

@article{beck2021deep,
  title={Deep splitting method for parabolic {P}{D}{E}s},
  author={Beck, Christian and Becker, Sebastian and Cheridito, Patrick and Jentzen, Arnulf and Neufeld, Ariel},
  journal={SIAM J. Sci. Comput.},
  volume={43},
  number={5},
  pages={A3135--A3154},
  year={2021},
  publisher={SIAM}
}

@incollection{raissi2024forward,
  title={Forward--backward stochastic neural networks: deep learning of high-dimensional partial differential equations},
  author={Raissi, Maziar},
  booktitle={Peter Carr Gedenkschrift: Research Advances in Mathematical Finance},
  pages={637--655},
  year={2024},
  publisher={World Scientific}
}

@article{beck2024overcome,
  title={Overcoming the curse of dimensionality in the numerical approximation of high-dimensional semilinear elliptic partial differential equations},
  author={Beck, Christian and Gonon, Lukas and Jentzen, Arnulf},
  journal={Part. Differ. Equ. Appl.},
  volume={5},
  number={6},
  pages={31},
  year={2024},
  publisher={Springer}
}

@article{hu2024tackling,
  title={Tackling the curse of dimensionality with physics-informed neural networks},
  author={Hu, Zheyuan and Shukla, Khemraj and Karniadakis, George Em and Kawaguchi, Kenji},
  journal={Neural Netw.},
  volume={176},
  pages={106369},
  year={2024},
  publisher={Elsevier}
}

@article{chiu2022can,
  title={{C}{A}{N}-{P}{INN}: A fast physics-informed neural network based on coupled-automatic--numerical differentiation method},
  author={Chiu, Pao-Hsiung and Wong, Jian Cheng and Ooi, Chinchun and Dao, My Ha and Ong, Yew-Soon},
  journal={Comput. Methods Appl. Mech. Engrg.},
  volume={395},
  pages={114909},
  year={2022},
  publisher={Elsevier}
}

@article{lv2023hybrid,
  title={A hybrid physics-informed neural network for nonlinear partial differential equation},
  author={Lv, Chunyue and Wang, Lei and Xie, Chenming},
  journal={Internat. J. Modern Phys. C},
  volume={34},
  number={06},
  pages={2350082},
  year={2023},
  publisher={World Scientific}
}

@article{pang2019fpinns,
  title={f{PINNs}: Fractional physics-informed neural networks},
  author={Pang, Guofei and Lu, Lu and Karniadakis, George Em},
  journal={SIAM J. Sci. Comput.},
  volume={41},
  number={4},
  pages={A2603--A2626},
  year={2019},
  publisher={SIAM}
}

@article{sirignano2018dgm,
  title={D{GM}: A deep learning algorithm for solving partial differential equations},
  author={Sirignano, Justin and Spiliopoulos, Konstantinos},
  journal={J. Comput. Phys.},
  volume={375},
  pages={1339--1364},
  year={2018},
  publisher={Elsevier}
}

@article{hu2025smooth,
  title={Bias-variance trade-off in physics-informed neural networks with randomized smoothing for high-dimensional PDEs},
  author={Hu, Zheyuan and Yang, Zhouhao and Wang, Yezhen and Karniadakis, George E and Kawaguchi, Kenji},
  journal={SIAM J. Sci. Comput.},
  volume={47},
  number={4},
  pages={C846--C872},
  year={2025},
  publisher={SIAM}
}

@article{hu2025score,
  title={Score-based physics-informed neural networks for high-dimensional {F}okker--{P}lanck equations},
  author={Hu, Zheyuan and Zhang, Zhongqiang and Karniadakis, George E and Kawaguchi, Kenji},
  journal={SIAM J. Sci. Comput.},
  volume={47},
  number={3},
  pages={C680--C705},
  year={2025},
  publisher={SIAM}
}

@article{shi2024stde,
  title={Stochastic taylor derivative estimator: Efficient amortization for arbitrary differential operators},
  author={Shi, Zekun and Hu, Zheyuan and Lin, Min and Kawaguchi, Kenji},
  journal={Adv. Neural Inf. Process. Syst.},
  volume={37},
  pages={122316--122353},
  year={2024}
}

@article{bolager2023sampling,
  title={Sampling weights of deep neural networks},
  author={Bolager, Erik L and Burak, Iryna and Datar, Chinmay and Sun, Qing and Dietrich, Felix},
  journal={Adv. Neural Inf. Process. Syst.},
  volume={36},
  pages={63075--63116},
  year={2023}
}

@article{rahimi2007random,
  title={Random features for large-scale kernel machines},
  author={Rahimi, Ali and Recht, Benjamin},
  journal={Adv. Neural Inf. Process. Syst.},
  volume={20},
  year={2007}
}

@article{li2024forlap,
  title={A computational framework for neural network-based variational {M}onte {C}arlo with {F}orward {L}aplacian},
  author={Li, Ruichen and Ye, Haotian and Jiang, Du and Wen, Xuelan and Wang, Chuwei and Li, Zhe and Li, Xiang and He, Di and Chen, Ji and Ren, Weiluo and others},
  journal={Nat. Mach. Intell.},
  volume={6},
  number={2},
  pages={209--219},
  year={2024},
  publisher={Nature Publishing Group UK London}
}

@article{peherstorfer2022breaking,
  title={Breaking the {K}olmogorov barrier with nonlinear model reduction},
  author={Peherstorfer, Benjamin},
  journal={Not. Am. Math. Soc.},
  volume={69},
  number={5},
  pages={725--733},
  year={2022},
  publisher={American Mathematical Society}
}

@article{berman2024randomized,
  title={Randomized sparse neural {G}alerkin schemes for solving evolution equations with deep networks},
  author={Berman, Jules and Peherstorfer, Benjamin},
  journal={Adv. Neural Inf. Process. Syst.},
  volume={36},
  year={2024}
}

@article{krishnapriyan2021characterizing,
  title={Characterizing possible failure modes in physics-informed neural networks},
  author={Krishnapriyan, Aditi and Gholami, Amir and Zhe, Shandian and Kirby, Robert and Mahoney, Michael W},
  journal={Adv. Neural Inf. Process. Syst.},
  volume={34},
  pages={26548--26560},
  year={2021}
}

@article{wang2021understanding,
  title={Understanding and mitigating gradient flow pathologies in physics-informed neural networks},
  author={Wang, Sifan and Teng, Yujun and Perdikaris, Paris},
  journal={SIAM J. Sci. Comput.},
  volume={43},
  number={5},
  pages={A3055--A3081},
  year={2021},
  publisher={SIAM}
}

@article{wang2022and,
  title={When and why {PINNs} fail to train: A neural tangent kernel perspective},
  author={Wang, Sifan and Yu, Xinling and Perdikaris, Paris},
  journal={J. Comput. Phys.},
  volume={449},
  pages={110768},
  year={2022},
  publisher={Elsevier}
}

@article{kast2024positional,
  title={Positional embeddings for solving {PDE}s with evolutional deep neural networks},
  author={Kast, Mariella and Hesthaven, Jan S},
  journal={J. Comput. Phys.},
  volume={508},
  pages={112986},
  year={2024},
  publisher={Elsevier}
}

@article{ainsworth2021gnn,
  title={Galerkin neural networks: A framework for approximating variational equations with error control},
  author={Ainsworth, Mark and Dong, Justin},
  journal={SIAM J. Sci. Comput.},
  volume={43},
  number={4},
  pages={A2474--A2501},
  year={2021},
  publisher={SIAM}
}

@article{ainsworth2025egnn,
  title={Extended {G}alerkin neural network approximation of singular variational problems with error control},
  author={Ainsworth, Mark and Dong, Justin},
  journal={SIAM J. Sci. Comput.},
  volume={47},
  number={3},
  pages={C738--C768},
  year={2025},
  publisher={SIAM}
}

@article{wang2025tensor,
  title={Tensor neural networks for high-dimensional {F}okker--{P}lanck equations},
  author={Wang, Taorui and Hu, Zheyuan and Kawaguchi, Kenji and Zhang, Zhongqiang and Karniadakis, George Em},
  journal={Neural Netw.},
  volume={185},
  pages={107165},
  year={2025},
  publisher={Elsevier}
}

@article{nesterov2012efficiency,
  title={Efficiency of coordinate descent methods on huge-scale optimization problems},
  author={Nesterov, Yu},
  journal={SIAM J. Optim.},
  volume={22},
  number={2},
  pages={341--362},
  year={2012},
  publisher={SIAM}
}

@article{Yu,
  title={Scaling optimized {H}ermite approximation methods},
  author={Hu, H. and Yu, H.},
  journal={SIAM J. Numer. Anal.},
  volume={64},
  pages={125--147},
  year={2026},
  publisher={SIAM}
}
